\documentclass[review]{elsarticle}

\journal{Elsevier}

\usepackage{graphicx}
\usepackage{caption}
\usepackage{subcaption}
\usepackage[ruled]{algorithm2e}
\usepackage{bm}
\usepackage{amsfonts}
\usepackage{epstopdf}
\usepackage{pgfplots}
\usepackage{amsmath}
\usepackage{mathtools}
\usepackage{amsthm}
\usepackage{acronym}

\theoremstyle{plain}
\newtheorem{thm}{Theorem} 

\newtheorem{assumption}{Assumption}

\newcommand\norm[1]{\left\lVert#1\right\rVert}

\def\s{{\mathbf s}}

\def\u{{\mathbf u}}
\def\x{{\mathbf x}}
\def\y{{\mathbf y}}
\def\z{{\mathbf z}}
\def\X{{X}}
\def\Y{{Y}}
\def\Z{{Z}}
\def\muu{\boldsymbol{\mu}}
\def\C{{\mathbf C}}
\def\J{{\mathbf J}}

\def\K{{\mathbf K}}
\def\W{{\mathbf W}}
\def\Q{{\mathbf Q}}
\def\phii{{\boldsymbol \phi}}

\newlength
\figureheight
\newlength
\figurewidth

\begin{document}

\begin{frontmatter}

\title{Multimodal Latent Variable Analysis}


\author[mymainaddress]{Vardan Papyan\corref{mycorrespondingauthor}}
\cortext[mycorrespondingauthor]{Corresponding author.}
\ead{vardanp@campus.technion.ac.il}

\author[mysecondaryaddress]{Ronen Talmon}

\address[mymainaddress]{Department of Computer Science, Technion -- Israel Institute of Technology, Israel}
\address[mysecondaryaddress]{Department of Electrical Engineering, Technion -- Israel Institute of Technology, Israel}

\begin{abstract}
Consider a set of multiple, multimodal sensors capturing a complex system or a physical phenomenon of interest. Our primary goal is to distinguish the underlying sources of variability manifested in the measured data. The first step in our analysis is to find the common source of variability present in all sensor measurements. We base our work on a recent paper, which  tackles this problem with alternating diffusion (AD). In this work, we suggest to further the analysis by extracting the sensor-specific variables in addition to the common source. We propose an algorithm, which we analyze theoretically, and then demonstrate on three different applications: a synthetic example, a toy problem, and the task of fetal ECG extraction.
\end{abstract}


\end{frontmatter}


\section{Introduction}
The analysis of a physical phenomenon or some complex system at hand can often be made easier through the use of several sensors instead of a single complex one. The hope is that each of the sensors captures a different part of the convoluted system, while the fusion of all the information captures the global picture. This line of thinking has led to the abundance of multimodal and multi-sensory data in recent years and to an increased demand for algorithms that enable its processing and analysis \cite{Lahat}. A prime example for the above is medical diagnosis based on collected bedside data, where one monitors a patient using various basic sensors, such as heart rate, pulse, blood pressure and oxigen level just to name a few, and attempts to diagnose the complex system at hand, that is the patient state, using the collected data.

Elaborate systems, such as the one mentioned above, are usually governed by many sources of variability. A central problem is then the analysis of latent sources, given measurements originating from several sensors of various types. Naturally, analyzing the measured data in terms of its underlying sources of variability requires their extraction. Unfortunately, driving sources are often hidden in nonlinear unknown manners, thereby posing a true challenge to the analysis and to the extraction.

In order to facilitate the extraction of the different sources of variability, we divide them into two conceptual categories: (i) sources of variability common to all sensors; and (ii) variables unique to a specific sensor. In our work, we focus on a two step implementation where we first reveal the common variable. Once it is found, we extract the remaining sources of variability, i.e, the sensor-specific ones. Intuitively, our approach marginalizes the common variable, which is found in the first step, and then continues to extract the sources of variability left in the filtered data. This simplifies our task, since we do not attempt to extract all the sources manifested in the data at once.

In this paper, we use an unsupervised manifold learning approach to address the problem. Various manifold learning algorithms were proposed in the literature over the years, e.g., \cite{belkin2001laplacian,donoho2003hessian,roweis2000nonlinear,tenenbaum2000global}. However, most of these classical methods assume that the data is captured by a single sensor, rather than in the multimodal multi-sensory setting we consider here. We focus on a particular paradigm -- the Diffusion Geometry, as presented in \cite{coifman2006diffusion,coifman2005geometric}. Using this framework, the \ac{AD} algorithm was recently proposed in \cite{lederman2015learning,talmon2016latent} for the purpose of extracting the source of variability common to multiple sensors. 
\ac{AD} follows a recent line of papers that propose to use multiplications and manipulations of kernels for the purpose of fusing data from different sensors, e.g., \cite{de_sa_spectral_2005,de_sa_multi_view_2010,boots_two_manifold_2012,lindenbaum2015multiview}. Similarly to recently presented nonlinear methods, e.g., \cite{michaeli2015nonparametric,yair2016local}, \ac{AD} is shown to reveal only the common components among all processed sensors. 
Successful applications of \ac{AD} to real measured data were demonstrated, e.g., in \cite{lederman2015alternating} for the task of sleep stage identification. Herein, we rely on \ac{AD} and aim to extend it by further analyzing the measurements and finding the sensor-specific variables. Our main motivation is that in some applications the sensor specific variables are far more important than the common variable. Indeed, we show one real-life example of such an application -- fetal \ac{ECG} extraction.

Our main contribution in this work is a novel algorithm, attempting to recover all the sources of variability manifested in a set of multi-sensory multimodal measurements. We justify our proposed scheme theoretically, showing that it is guaranteed to find the underlying parametrizations under certain prescribed conditions. In addition, we demonstrate its applicability in three different applications: a synthetic example, a toy problem and a real-life application.

This paper is organized as follows. In Section \ref{Sec:ProblemFormulation} we introduce formally the problem we address, and in Section \ref{Sec:Preliminaries} we review the diffusion maps and \ac{AD} algorithms. In Section \ref{Sec:ProposedMethod} we present the proposed method and in Section \ref{Sec:TheoreticalAnalysis} we analyze it theoretically. In Section \ref{Sec:Experiments} we test our method on a synthetic example, a toy problem and a real-life application -- the extraction of fetal \ac{ECG}. We conclude this paper in Section \ref{Sec:Conclusions}.

\section{Problem Formulation} \label{Sec:ProblemFormulation}

Consider three latent random variables $\X$, $\Y$ and $\Z$ in $\mathbb{R}^{d_x}$, $\mathbb{R}^{d_y}$ and $\mathbb{R}^{d_z}$, respectively, which are jointly distributed according to some \ac{PDF} denoted by $P(\X,\Y,\Z)$. Following the work in \cite{lederman2015learning}, we assume that the variables $\Y$ and $\Z$ are independent given $\X$, i.e., the joint \ac{PDF} can be written as follows:
\begin{equation}
P(\X,\Y,\Z) = P(\Y|\X)P(\Z|\X)P(\X),
\label{eq:pdf_fac}
\end{equation}
where $P(\X)$ is the marginal \ac{PDF} of $\X$, and $P(\Y|\X)$ and $P(\Z|\X)$ are the conditional \acp{PDF} of $\Y$ and $\Z$ given $\X$, respectively.
When measuring a system of interest, a measurement instance is defined by the triplet $(\x_i,\y_i,\z_i)$, which is a realization sampled from $P(\X,\Y,\Z)$. We do not have access to the latent variables; instead, we have two sensors observing the system at hand through two unknown observation functions given by $g(\x_i,\y_i)$ and $h(\x_i,\z_i)$. We assume $g$ and $h$ are smooth and locally invertible bilipschitz functions.
Let $\{\s_i^{(1)}\}_{i=1}^N$ and $\{\s_i^{(2)}\}_{i=1}^N$ denote two sets of $N$ measurement samples, taken simultaneously from the two sensors, such that $\s_i^{(1)}=g(\x_i,\y_i)\in\mathbb{R}^{d_1}$ and $\s_i^{(2)}=h(\x_i,\z_i)\in\mathbb{R}^{d_2}$, where $\{ (\x_i,\y_i,\z_i) \}_{i=1}^N$ are $N$ realizations of the system's hidden variables. 
In other words, we have hidden realizations $(\x_i,\y_i,\z_i)$ of three underlying variables and two sensor observations $\s_i^{(1)}$ and $\s_i^{(2)}$; $\x_i$ is the common latent variable between the two observations, whereas $\y_i$ and $\z_i$ are two sensor-specific variables.

Given the two sets of measurement samples, the work in \cite{lederman2015learning} showed that a method based on \ac{AD} operators extracts a parameterization of the common variable $X$. In this work, we aim to further the analysis and extract a parametrization of the variables $\Y$ and $\Z$ as well. Such a complementing capability enables us to fully parametrize all the hidden variables underlying the measurements of the system of interest.

Although the analysis and methods used in this paper will be carried out from a different standpoint, the factorization in \eqref{eq:pdf_fac} can be used to explain the main concept. Intuitively, the extraction of the common variable $\X$ in \cite{lederman2015learning} can be viewed as a marginalization operator applied to the joint probability $P(\X,\Y,\Z)$ obtaining $P(\X)$. In this work, we devise another operator which uses $P(\X)$ to construct the conditional probabilities $P(\Y|\X)$ and $P(\Z|\X)$. Then, given $P(\Y|\X)$ and $P(\Z|\X)$, it marginalizes the variable $\X$ and obtains a parametrization of the sensor-specific variables $\Y$ and $\Z$.

\section{Preliminaries}
\label{Sec:Preliminaries}
\subsection{Diffusion Maps}
\label{subsec:DiffusionMaps}

Diffusion maps \cite{coifman2006diffusion,coifman2005geometric} is a data-driven nonlinear dimensionality reduction algorithm. Given a set of $N$ measurements $\{\u_i\}_{i=1}^N$ in $\mathbb{R}^m$, the method commences by constructing an affinity matrix $\W$ of size $N \times N$, whose $(i,j)$-th entry is given by
\begin{equation}
W_{i,j} = \exp\left( -\frac{\norm{\u_i - \u_j}^2}{\epsilon} \right), \ \forall i,j=1,\dots,N.
\end{equation}
Intuitively, $\W$ can be interpreted as a weight matrix of a graph with $N$ vertices, where the coefficient $\epsilon>0$ dictates the sparsity of the edges. If $\epsilon$ is small, most edges have a negligible, close to zero weight and the graph is effectively sparse, whereas if $\epsilon$ is large, most edges are assigned with non negligible weights and the graph is dense.

The constant $\epsilon$ is usually chosen according to the data at hand, and in this work we set it using the method suggested in \cite{lederman2015learning}. Therein, the constant was chosen to be $\epsilon = \sqrt{\epsilon_i \epsilon_j}$, where $\epsilon_i$ a scaling constant corresponding to the $i$-th vertex. In particular, $\epsilon_i$ is chosen to be the mean squared distance from the $i$-th vertex to its $k$ nearest neighbors.

The next step is to normalize the affinity matrix $\W$, which results in the matrix $\K$. Various normalization procedures have been suggested in the literature \cite{nadler2006diffusion,lafon2004diffusion}, each having a different interpretation when analyzed theoretically. In this work, $\K$ is constructed by dividing each column of $\W$ by its sum, yielding a column-stochastic matrix. As a result, $\K$ can be viewed as a transition probability matrix of a Markov chain on the graph.

Once the affinity matrix is constructed and normalized, a $d$-dimensional embedding $\{ \hat{\u}_i \}_{i=1}^N$ is formed according to the following nonlinear map:
\begin{equation} \label{eq:parametrization}f
\hat{\u}_i = \left[ \lambda_1^m \phi_1^i, \dots, \lambda_d^m \phi_d^i \right]^T,
\end{equation}
where $\phii_j$ is the $j$-th left eigenvector of the matrix $\K$ and $\phi_j^i$ is its $i$-th entry, $\lambda_j^m$ is the $j$-th eigenvalue (when the eigenvalues are denoted in descending order) raised to the power of $m$, and $m>0$ is a constant. 
Typically, $d$ is set to be much smaller than $\min(d_1,d_2)$, thereby attaining dimensionality reduction. 
In addition to providing compact representation, this nonlinear map attempts to reveal the essence of the data in few dimensions, accurately representing their underlying intrinsic variables. 
In the context of diffusion maps, special attention is given to the Euclidean distance between the embedded samples $\hat{\u}_i$.
Specifically, the Euclidean distance between the embedded samples $\hat{\u}_i$ approximates the Euclidean distance between the corresponding columns of $\K^m$. This distance is termed the {\em diffusion distance}, since it takes into account transition probabilities on the constructed graph consisting of $m$ Markov chain steps.
We note that diffusion distance plays a large role in the algorithm presented in this paper.
For more details, as well as the motivation behind this particular dimensionality reduction method, we refers the reader to \cite{coifman2006diffusion}.

\subsection{Alternating Diffusion}

Given two sets of measurement samples originating from two sensors, i.e., $\{\s_i^{(1)}\}_{i=1}^N$ and $\{\s_i^{(2)}\}_{i=1}^N$, the first step in the \ac{AD} algorithm is constructing two pairwise affinity matrices, $\W^{(1)}$ and $\W^{(2)}$ based on the Gaussian kernel
\begin{align}
& W_{i,j}^{(1)} = \exp\left( -\norm{\s_i^{(1)} - \s_j^{(1)}}^2 / \epsilon^{(1)}\right) \label{eq:k1} \\
& W_{i,j}^{(2)} = \exp\left( -\norm{\s_i^{(2)} - \s_j^{(2)}}^2 / \epsilon^{(2)}\right) \label{eq:k2}.
\end{align}
The constants $\epsilon^{(1)}$ and $\epsilon^{(2)}$ have a similar interpretation to the one presented in Section \ref{subsec:DiffusionMaps}. The algorithm proceeds by normalizing $\W^{(1)}$ and $\W^{(2)}$ to be column-stochastic, yielding two matrices $\K^{(1)}$ and $\K^{(2)}$, where the sum of each of their columns equals one. As a result, each stochastic matrix can be interpreted as a transition probability matrix of a Markov chain on a graph whose vertices are the samples (as described in Section \ref{subsec:DiffusionMaps}). In other words, the $(i,j)$-th entry in $\K^{(1)}$ or in $\K^{(2)}$ represents the probability of transition to the $i$-th vertex from the $j$-th vertex in the graph. Importantly, by construction \eqref{eq:k1}, $\K^{(1)}$ describes a Markov chain that jumps in high probability from the $j$-th vertex to the $i$-th vertex if the underlying values of both $\X$ and $\Y$ are similar (namely, $\x_i$ is similar to $\x_j$ and $\y_i$ is similar to $\y_j$). Analogously, by construction \eqref{eq:k2}, $\K^{(2)}$ describes a Markov chain that jumps in high probability from the $j$-th vertex to the $i$-th vertex if the underlying realizations of both $\X$ and $\Z$ are similar.

Given the normalized matrices $\K^{(1)}$ and $\K^{(2)}$, an \ac{AD} kernel is then defined by
\begin{equation}
\K = \K^{(2)} \K^{(1)}.
\end{equation}
This corresponds to a transition probability matrix $\K$ consisting of two consecutive, {\em alternating} steps -- the first step is employed according to $\K^{(1)}$ and second according to $\K^{(2)}$. Next, by raising $\K$ to the power of $m$, we obtain a transition matrix $\K^m$ that corresponds to $2m$ steps, where the odd steps correspond to $\K^{(1)}$ and the even steps correspond to $\K^{(2)}$. Consequently, the odd steps jump (in high probability) to a vertex where both $\X$ and $\Y$ values are similar, while the even steps jump (in high probability) to a vertex where both $\X$ and $\Z$ are similar. As a result, after many (odd and even) steps, we maintain similarity only to the $\X$ value whereas the $\Y$ and $\Z$ may vary significantly.

In order to obtain an affinity matrix in terms of the common variable $\X$ between pairs of samples $(\s_i^{(1)},\s_i^{(2)})$ and $(\s_j^{(1)},\s_j^{(2)})$ (recalling that each pair of samples shares the same $\X$ value according to the model assumptions presented in Section \ref{Sec:ProblemFormulation}), the method computes the $\ell_2$ distance between the corresponding columns in the matrix $\K^m$. In \cite{lederman2015learning} a rigorous analysis is provided justifying this statement. Moreover, it was suggested to use a refinement step consisting of an additional diffusion maps application where the columns of $\K^m$ are the new graph vertices, resulting in a low dimensional embedding as defined in \eqref{eq:parametrization}. In the \ac{AD} setting, since the underlying variable of the affinity matrix $\K^m$ is $\X$, we shall denote the resulting embedding by $\hat{\x}$ instead of the general notation of $\hat{\u}$, which was used in \eqref{eq:parametrization}.

\section{Proposed Method}
\label{Sec:ProposedMethod}

The first step towards a full parametrization of all the latent variables underlying the measurements is finding the common latent variable $\X$, as previously suggested, using the \ac{AD} algorithm. Once the common variable is extracted, we proceed to analyzing the measurements from the first and second sensors separately. Hereafter, for the sake of brevity, we will focus on the analysis of the first sensor only, while the analysis of the second sensor is analogous. For each sample $\s_i^{(1)}$, let $\mathcal{N}_i^{(1)}$ be a neighborhood of samples consisting of samples $j$ with a similar common variable. Formally, define
\begin{equation}
	\mathcal{N}_i^{(1)} = \left\{ j \ | \ \| \hat{\x}_j - \hat{\x}_i \|_2 < \eta_i \right\},
\end{equation}
where $\eta_i>0$ is a small tunable threshold. In practice, instead of fixing a threshold $\eta_i$ for every signal, we choose all the neighborhoods $\mathcal{N}_i$ to be of the same size $q$. In other words, the $\eta_i$ are fixed implicitly such that the size of each neighborhood $\mathcal{N}_i$ is equal to $q$. The key point in defining these neighborhoods relies on the assumption that the \ac{AD} algorithm is able to successfully recover the common variable $\X$ and to suppress the sensor specific variable $\Y$. That is, the measurements in these neighborhoods, i.e., $\left\{ \s_j^{(1)} | j\in \mathcal{N}_i^{(1)} \right\}$, share equal, or close, values of $\X$. As a result, the only remaining variability in such neighborhoods of samples stems from variations in $\Y$. 

For each such neighborhood, we propose to compute its sample mean
\begin{equation} \label{eq:mean}
\muu_i^{(1)}=\frac{1}{\left|\mathcal{N}_i^{(1)}\right|}\sum_{j\in \mathcal{N}_i^{(1)}}\s_j^{(1)},
\end{equation}
and its sample covariance
\begin{equation} \label{eq:covariance}
\C_i^{(1)}=\frac{1}{\left|\mathcal{N}_i^{(1)}\right|}\underset{{j\in \mathcal{N}_i^{(1)}}}{\sum}(\s_j^{(1)}-\muu_i^{(1)})(\s_j^{(1)}-\muu_i^{(1)})^T
\end{equation}
both in the domain of the measurements $\{ \s_j^{(1)} \}$. Thus, $(\muu_i^{(1)},\C_i^{(1)})$ can be seen as a Gaussian representation of the local variability of the sensor-specific variable $\Y$ around every sample, and hence, in light of the discussion above, we have a local representation of $\Y$.

In order to get a global parametrization, we compare the local neighborhoods by means of ``registration'' of point clouds or Gaussian distributions. Consider the following affinity kernel
\begin{align} \label{eq:mahalanobis_distance}
&\widetilde{W}_{i,j}^{(1)} = \text{exp}\left\{-\frac{1}{\epsilon}
\left((\s_i^{(1)} - \muu_i^{(1)}) - (\s_j^{(1)} -\muu_j^{(1)})\right)^T
\left({\C_i^{(1)}}^\dagger+{\C_j^{(1)}}^\dagger\right) \right. \\
&\phantom{............................}
\bigg. \left((\s_i^{(1)} - \muu_i^{(1)}) - (\s_j^{(1)} - \muu_j^{(1)})\right)
\bigg\}.
\end{align}
We have denoted by ${\C_i^{(1)}}^\dagger$ the Moore-Penrose pseudoinverse, which is employed since the rank of the covariance matrix is lower than its dimension. This follows the underlying assumption that the dimension of the measurements $d_1$ is larger than the dimension of the sensor-specific variable $d_y$. We note that the omission of the low eigenvalues and their corresponding eigenvectors, as done by the pseudoinverse, results both in denoising possible ambient noise and estimation inaccuracies, and also in the attenuation of the common variable $\X$ remainders, thereby enhancing the desired variation -- that of the sensor-specific variable only.

\begin{algorithm}[t]
\textbf{Input}: signals $\{\s_i^{(1)}\}_{i=1}^N$ and $\{\s_i^{(2)}\}_{i=1}^N$ originating from both sensors.\newline
\textbf{Output}: parametrizations of the sensor specific variables $\Y$ and $\Z$.
\begin{enumerate}
	\item Compute the parametrization $\hat{\x}$ of the common variable $\X$ using the alternation-diffusion algorithm.
	\item For each signal $\s_i^{(1)}$:
	\begin{enumerate}
		\item Find the local neighborhood of $\s_i^{(1)}$ denoted by $\mathcal{N}_i^{(1)}$ in terms of the parametrization found in the previous step.
		\item Compute the local mean, using Equation \eqref{eq:mean}, and the local Covariance, using Equation \eqref{eq:covariance}.
	\end{enumerate}
	\item Compute the affinity matrix using the Mahalanobis distance between the signals $\s_i^{(1)}$ and $\s_j^{(1)}$, as done in Equation \eqref{eq:mahalanobis_distance}.
	\item Apply the standard diffusion maps algorithm on the above matrix to obtain a parametrization $\hat{\y}$ for the variable $\Y$.
	\item Repeat the above steps for the second sensor.
\end{enumerate}

\caption{The proposed algorithm.}
\label{alg1}
\end{algorithm}

The distance in the Gaussian kernel in \eqref{eq:mahalanobis_distance} is a modified Mahalanobis distance between the signal samples $\s_i^{(1)}$ and $\s_j^{(1)}$, which was presented in \cite{singer2008non}, with the exception that the local covariance matrices are computed based on neighborhoods in the extracted common variable domain. In \cite{talmon2013empirical,talmon2015intrinsic,talmon2015manifold}, this distance was used in the context of manifold learning and diffusion maps to determine an intrinsic representation of (single) sensor data, invariant to interferences and measurement modalities.
Such a manipulation of the Mahalanobis distance via the neighborhood choice was suggested in \cite{mishne2015graph} for the task of sea mine detection in sonar images, and in \cite{dsilva2016data} of reduction of stochastic dynamical systems. There, by controlling the locality within a pre-defined training set, a new metric, which is invariant to perturbations in the appearance of the target, was presented. In our work, rather than building invariances, we use a similar approach by appropriately choosing the neighborhoods in a multi-sensor setting in order to obtain a full parametrization of all the underlying sources of variability.

Once the affinity kernel $\widetilde{\W}^{(1)}$ is constructed, given that it captures only the variability of the (desired) sensor-specific variable $\Y$, we apply the standard diffusion maps algorithm in order to find the parametrization of the underlying variable $\Y$, denoted by $\hat{\y}$. The proposed algorithm is summarized in Algorithm \ref{alg1}.

\section{Theoretical Analysis}
\label{Sec:TheoreticalAnalysis}
In this section, we provide a theoretical analysis, showing that indeed the proposed algorithm approximates the distance between two signal samples in terms of the sensor-specific variable. As above, without loss of generality, we focus on signal samples arising from the first sensor, and therefore, our goal is to extract the variable $\Y$. For simplicity, in this section we omit the sensor index. A similar derivation to the one presented in this section was done in \cite{singer2008non} and in \cite{mishne2015graph}. Here, we highlight the significant differences both in terms of the analysis and in terms of the underlying assumptions.

\begin{assumption}\label{as:1}
If the measurement sample $\s_j=g(\x_j,\y_j)$ belongs to the neighborhood of $\s_i=g(\x_i,\y_i)$, i.e., $j \in \mathcal{N}_i^{(1)}$, then $\| \x_j - \x_i \|_2 = O( \| \y_j - \y_i \|_2^2)$.
\end{assumption}
This assumption relies on the ability of \ac{AD} to capture the common variable $\X$, as was proven in \cite{lederman2015learning}. By definition, if a signal sample $\s_j$ is in the neighborhood of $\s_i$, then the distance between their extracted values of common variable $X$ is small (which in practice, is controlled by the tunable threshold $\eta_i$).
Here we further assume that, if a signal sample $\s_j$ is in the neighborhood of $\s_i$, then the distance between their respective $X$ values (which is small by definition) is also smaller than the distance between their associated $Y$ values by at least one order of magnitude.  

\begin{assumption}\label{as:2}
Locally, for every signal sample $\s_i$, the empirical covariance matrix of the sensor-specific variable $\Y$ given the extracted common variable $\X$ is isotropic, i.e., it is given by
\begin{align}
\sum_{j\in\mathcal{N}_i^{(1)}} (\y_j - \y_i)(\y_j - \y_i)^T = \mathbf{I},
\end{align}
where $\mathbf{I}$ is the identity matrix.
\end{assumption}
While Assumption \ref{as:2} may seem to be artificial and restrictive, in Section \ref{Sec:Experiments}, we present experimental results supporting it empirically. In addition, we note that it was used in slightly different contexts in \cite{singer2008non,talmon2015manifold,tipping1999probabilistic} and successfully applied in many applications with real measured data. The following result follows Assumption \ref{as:1} and Assumption \ref{as:2}.

\begin{thm} \label{Thm:1}
For any signal sample $\s_i=g(\x_j,\y_j)$, if Assumption \ref{as:1} and Assumption \ref{as:2} are satisfied, then 
\begin{equation}
\sum_{j\in\mathcal{N}_i} (\s_j-\s_i)(\s_j-\s_i)^T = \J_i^y{\J_i^y}^T + O \left( \norm{\y_j - \y_i}_2^3 \right),
\end{equation}
where $\J_i^y$ is the Jacobian of the function $g$ with respect to the variables $\Y$, computed at the $i$-th sample $\s_i$.
\end{thm}

\begin{proof}
Using Taylor expansion, we can linearly approximate the observation function $g(\x_j,\y_j)$ around the point $(\x_i,\y_i)$, obtaining
\begin{align} \label{eq:taylor_approx_xy}
\s_j - \s_i & = \J_i^x (\x_j - \x_i) + \J_i^y (\y_j - \y_i) \nonumber \\
& + O \left( \norm{ \x_j - \x_i }_2^2 + \norm{ \y_j - \y_i }_2^2 + \left( \x_j - \x_i \right)^T \left( \y_j - \y_i \right) \right),
\end{align}
where $\J_i^x$ and $\J_i^y$ are the Jacobians at the $i$-th sample, $\s_i$, with respect to the variables $\X$ and $\Y$, respectively. The last term in \eqref{eq:taylor_approx_xy} encapsulates all the higher order derivatives that do not appear in this linear approximation.
Under Assumption \ref{as:1}, \eqref{eq:taylor_approx_xy} can be rewritten as
\begin{align} \label{eq:taylor_approx_x}
\s_j - \s_i & = \J_i^y (\y_j - \y_i) + O \left( \norm{\y_j - \y_i}_2^2 \right).
\end{align}
for any $j \in \mathcal{N}_i$.
Using \eqref{eq:taylor_approx_x} and by Assumption \ref{as:2}, the empirical covariance around the sample $\s_i$ is given by
\begin{align} \label{eq:expectation2}
 \sum_{j\in\mathcal{N}_i} (\s_j-\s_i)(\s_j-\s_i)^T = & \sum_{j\in\mathcal{N}_i} \J_i^y (\y_j - \y_i)(\y_j - \y_i)^T{\J_i^y}^T + O \left( \norm{\y_j - \y_i}_2^3 \right) \\
= & \ \J_i^y{\J_i^y}^T + O \left( \norm{\y_j - \y_i}_2^3 \right),
\end{align}
concluding our proof.
\end{proof}
Theorem \ref{Thm:1} shows that in order to estimate empirically the Gram matrix $\J_i^y{\J_i^y}^T$ of the Jacobian $\J_i^y$, we can simply compute the empirical covariance matrix of the samples in the neighborhood of $\x_i$, where the neighborhood is defined by the AD metric. This is accomplished without the knowledge of the function $g$ itself. In the next result, we use this Gram matrix for estimating the distance between a pair of signal samples in terms of the sensor-specific variable $\Y$.

\begin{thm} \label{Thm:2}
For any two signal sample $\s_i=g(\x_i,\y_i)$ and $\s_j=g(\x_j,\y_j)$, the Euclidean distance between the corresponding realizations of the sensor-specific variable is given by
\begin{align}
\norm{\y_j - \y_i}_2^2 & = (\s_j - \s_i)^T \left( \J_i^y {\J_i^y}^T \right)^\dagger (\s_j - \s_i) \\
& + O(\norm{\s_j - \s_i}_2^3).
\end{align}
\end{thm}

\begin{proof}
In the proof of Theorem \ref{Thm:1}, we considered the Taylor expansion of the observation function $g$ from the domain of the latent variables $\X$ and $\Y$ to the range of the measured signal. Similarly, consider the Taylor expansion of its inverse function $g^{-1}$ (recalling that $g$ is assumed bilipschitz), which is given by
\begin{equation}
\begin{bmatrix}
\x_j - \x_i \\
\y_j - \y_i
\end{bmatrix}
= \begin{bmatrix}
\Q_i^x \\
\Q_i^y
\end{bmatrix}
(\s_j - \s_i) + O(\norm{\s_j - \s_i}_2^2),
\end{equation}
where $\Q_i^x$ and $\Q_i^y$ are the Jacobian matrices of $g^{-1}$ with respect to the variables $\X$ and $\Y$, respectively. 
Isolating the $\Y$ variable yields
\begin{equation} \label{eq:taylor_approx_s}
\y_j - \y_i = \Q_i^y (\s_j - \s_i) + O(\norm{\s_j - \s_i}_2^2).
\end{equation}
By applying the $\ell_2$ norm to both sides, we obtain
\begin{align} \label{eq:l2_diff_KK}
\norm{\y_j - \y_i}_2^2 & = (\s_j - \s_i)^T {\Q_i^y}^T \Q_i^y (\s_j - \s_i) + O(\norm{\s_j - \s_i}_2^3).
\end{align}
The Taylor expansions in \eqref{eq:taylor_approx_x} and \eqref{eq:taylor_approx_s} correspond to $g$ and $g^{-1}$, respectively. 
Thus, due to the inverse function theorem, we have
\begin{equation}\label{eq:pinv}
{\Q_i^y}^T \Q_i^y = \left( \J_i^y {\J_i^y}^T \right)^{-1} .
\end{equation}
Typically, the dimension of the measurements is larger than the sum of the dimensions of the common and sensor-specific variables, i.e., $d_1 > d_x+d_y$. As a result, the number of rows in $\J_i^y$ is larger than the number of columns, and hence, the Gram matrix $\J_i^y {\J_i^y}^T$ is not full-rank. Consequently, it is not invertible and one needs to employ a pseudo-inverse operator instead. 
By substituting \eqref{eq:pinv} into \eqref{eq:l2_diff_KK}, we obtain
\begin{align}
\norm{\y_j - \y_i}_2^2 & = (\s_j - \s_i)^T \left( \J_i^y {\J_i^y}^T \right)^\dagger (\s_j - \s_i) + O(\norm{\s_j - \s_i}_2^3),
\end{align}
as required.
\end{proof}
Two important notes are due at this point.
One is that the above analysis is based on the Taylor expansion around the sample $\s_i$. If we repeat the derivations with the Taylor expansion around the sample $\s_j$ as well, then, the mean of the two resulting expressions is given by
\begin{align} \label{eq:final_exppresion}
\norm{\y_j - \y_i}_2^2 & = \frac{1}{2}(\s_j - \s_i)^T \left( \left( \J_i^y {\J_i^y}^T \right)^\dagger + \left( \J_j^y {\J_j^y}^T \right)^\dagger \right) (\s_j - \s_i) \\
& + O(\norm{\s_j - \s_i}_2^3).
\end{align}
Once such symmetrization is employed, further analysis presented in \cite{singer2008non} improves the order of the error term in \eqref{eq:final_exppresion} to $\norm{\s_j - \s_i}_2^4$.
Two is that the above analysis assumes (unrealistically) that every signal has a local mean equal to zero, i.e., $\muu_i=\mathbf{0}$. However, as presented in \cite{talmon2015manifold}, a similar derivation can be done without such an assumption.
Combining these two notes results in the expression presented in \eqref{eq:mahalanobis_distance}, without a rigorous proof.

To conclude, the analysis presented in this section coincides with the proposed method. Indeed, in Algorithm \ref{alg1}, we begin by seeking for signals close in terms of the common variable using the parametrization obtained from the \ac{AD} algorithm. Once such a neighborhood is found, we compute its local empirical covariance matrix \eqref{eq:covariance}, which is then used in the modified Mahalanobis distance \eqref{eq:mahalanobis_distance} to approximate the desired Euclidean distance. Theorem \ref{Thm:1} proves that the aforementioned empirical covariance approximates the Gram matrix $\J_i^y {\J_i^y}^T$. According to Theorem \ref{Thm:2}, this approximation of the Gram matrix can be used to approximate the Euclidean distance in terms of the desired sensor-specific variable $\Y$ via the modified Mahalanobis distance \eqref{eq:final_exppresion} (which is used in Algorithm \ref{alg1}).

Final remark concerns the accuracy of the Euclidean distances approximation. Theorem \ref{Thm:2} implies that when the distances are large, the error terms are large and the (local) approximation via the linear terms is poor.
This problem is ``automatically'' alleviated by the standard usage of the Gaussian kernel in \eqref{eq:mahalanobis_distance}; due to its fast decay, large distances are implicitly attenuated.

\section{Experimental Results}
\label{Sec:Experiments}
\subsection{Synthetic Example}
Consider three independent and identically distributed random variables, $\X$, $\Y$, $\Z$, sampled uniformly in $[0,1]$. We generate from these variables $3000$ triplets of $(\x_i,\y_i,\z_i)$.
Assume two sensors observing these hidden samples through the following nonlinear functions $g$ and $h$
\begin{align}
\s^{(1)}_i=g(\x_i,\y_i)=
\begin{bmatrix}
R+r^{(1)}cos(2\pi \y_i)cos(2\pi \x_i)\\
R+r^{(1)}cos(2\pi \y_i)sin(2\pi \x_i)\\
r^{(1)}sin(2\pi \y_i)
\end{bmatrix}\phantom{.}
\end{align}
and
\begin{align}
\s^{(2)}_i=h(\x_i,\z_i)=
\begin{bmatrix}
R+r^{(2)}cos(2\pi \z_i)cos(2\pi \x_i)\\
R+r^{(2)}cos(2\pi \z_i)sin(2\pi \x_i)\\
r^{(2)}sin(2\pi \z_i)
\end{bmatrix},
\end{align}
so that we obtain $3000$ pairs of signal measurements $(\s^{(1)}_i,\s^{(2)}_i)$.
We set $R=10$, $r^{(1)}=4$, $r^{(2)}=2$. 
Notice that $g$ and $h$ correspond to two tori, with major angle $\X$, serving as their common hidden variable, and with minor angles, $\Y$ or $\Z$, serving as their respective sensor-specific variables. 
We apply Algorithm \ref{alg1} to these samples where the size of the neighborhoods in the common variable domain is $q=11$. In Figure \ref{tori} we color both tori according to the extracted parametrization of the common variable and also according to the obtained parametrizations of the two sensor-specific variables. Indeed, we observe that our method accurately extracts the three hidden variables. The coloring of the tori according to the common variable $\X$ is highly coherent with the major angle, while the coloring with respect to the sensor-specific variables, $\Y$ and $\Z$, are consistent with the minor angles.

\begin{figure}[t]
    \centering
    \begin{subfigure}[h]{0.45\textwidth}
        \centering
       	\includegraphics[width=5cm, height=4cm]{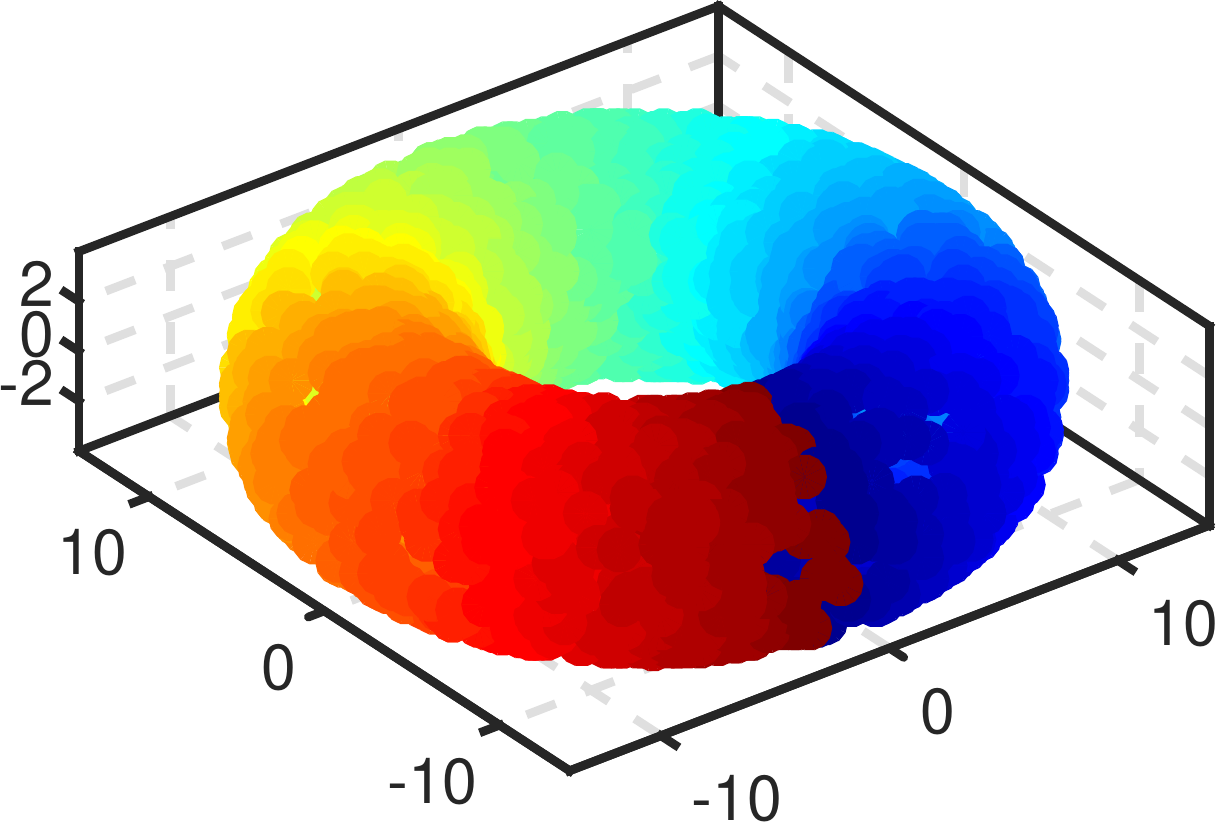}
    \end{subfigure}
    \qquad
	\centering
    \begin{subfigure}[h]{0.45\textwidth}
        \centering
        \includegraphics[width=5cm, height=4cm]{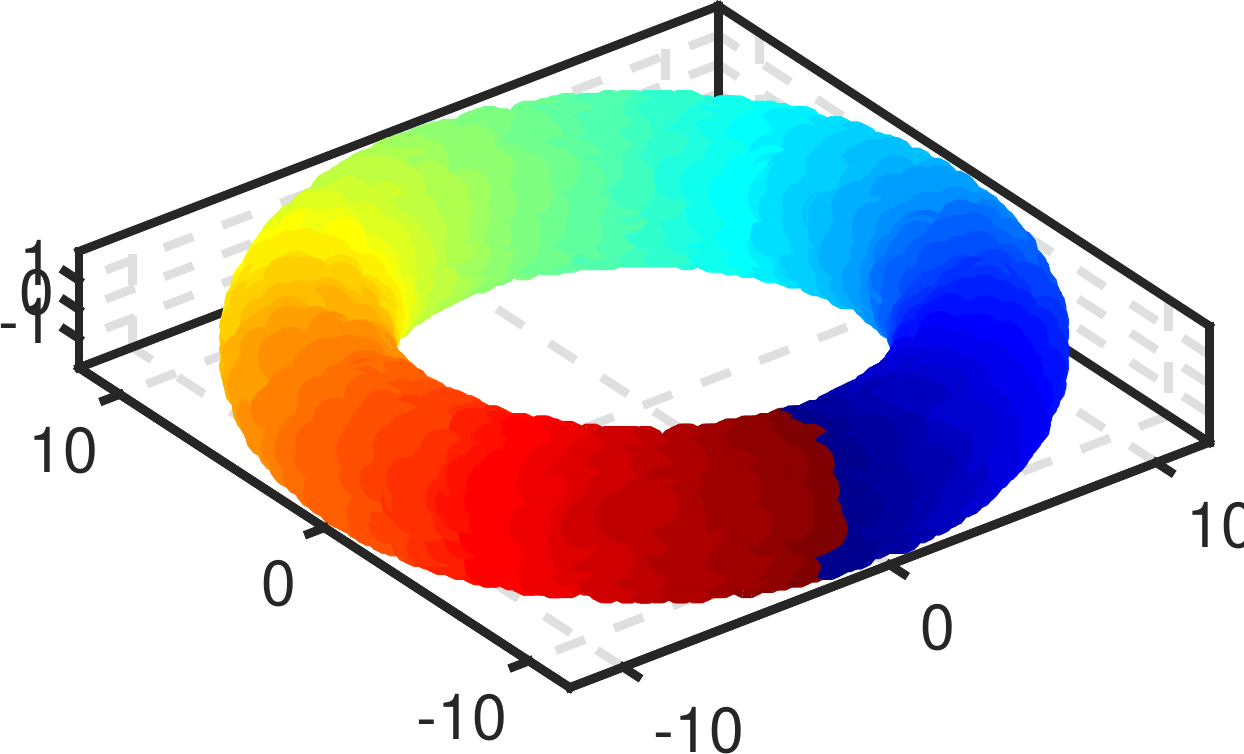}
    \end{subfigure}
    \qquad
    \vfill
    \centering
    \begin{subfigure}[h]{0.45\textwidth}
        \centering
        \includegraphics[width=5cm, height=4cm]{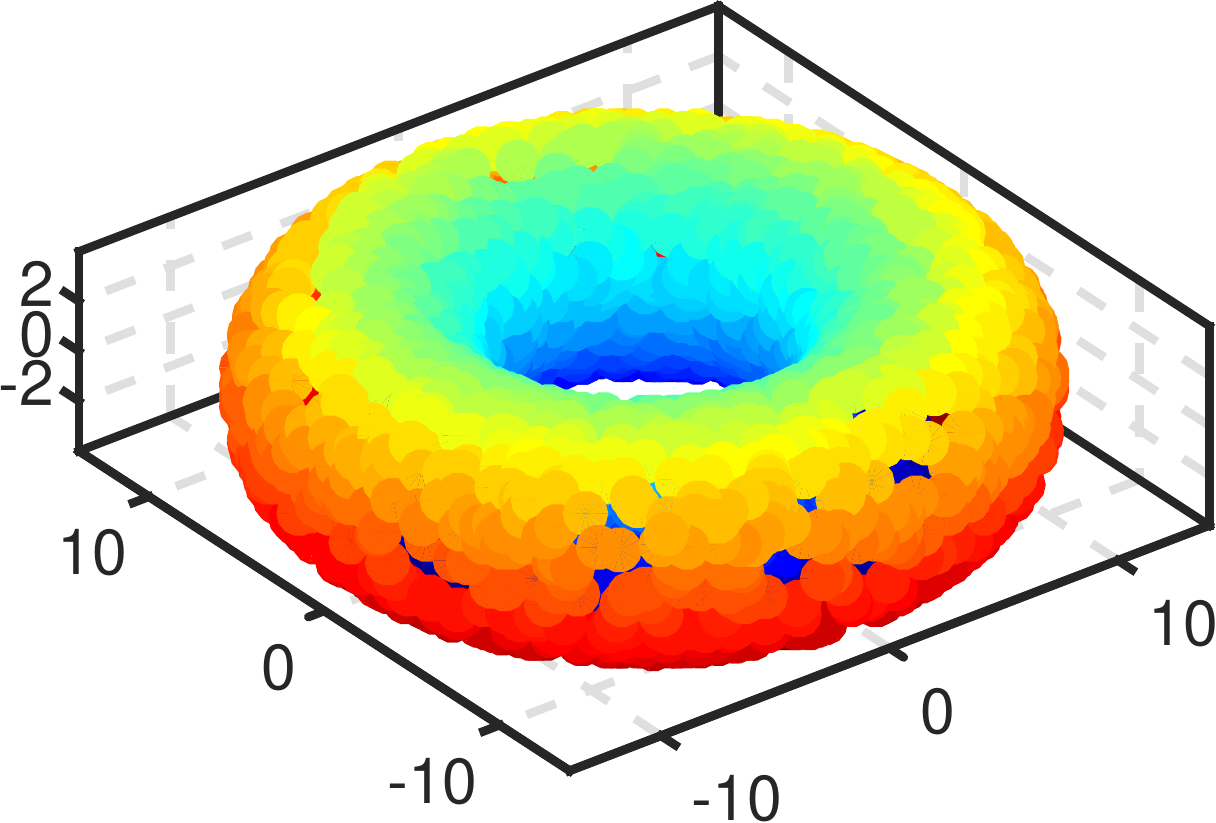}
    \end{subfigure}
    \qquad
    \centering
    \begin{subfigure}[h]{0.45\textwidth}
        \centering
        \includegraphics[width=5cm, height=4cm]{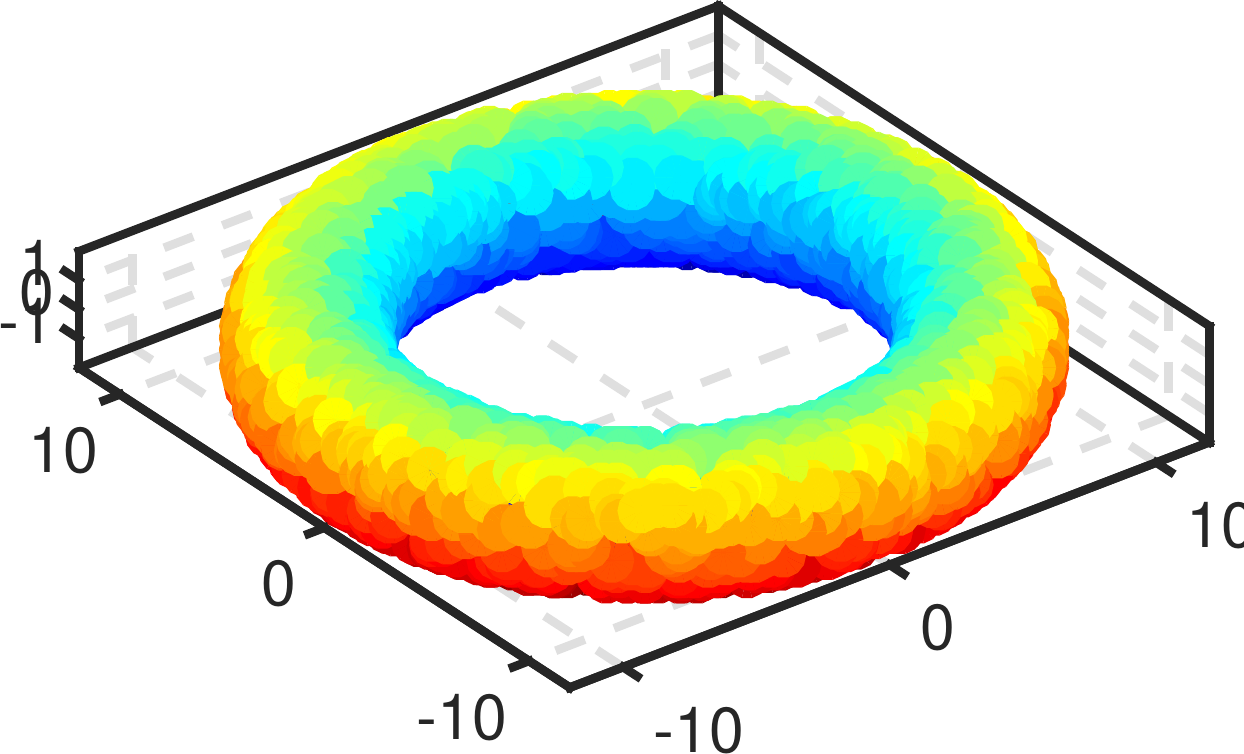}
    \end{subfigure}
    \caption{On the left we plot the samples $\s^{(1)}_i$ and on the right the samples $\s^{(2)}_i$. In the top row, the samples are colored according to the obtained parametrization of the common variable, $\X$. In the bottom row, the samples are colored according to the respective parametrization obtained for the sensor-specific variable, $\Y$ and $\Z$.}
    \label{tori}
\end{figure}

\begin{figure}[t]
    \centering
	\begin{subfigure}[h]{0.45\textwidth}
        \centering
        \includegraphics[width=1\textwidth]{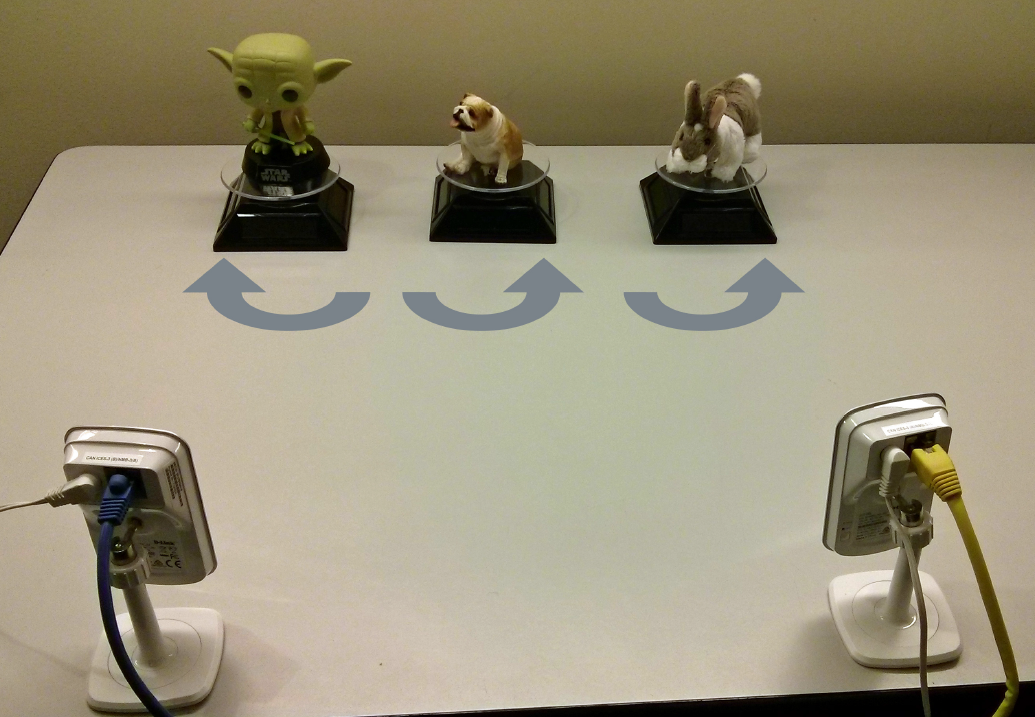}
    \end{subfigure}
    \begin{subfigure}[h]{0.205\textwidth}
        \centering
        \includegraphics[width=1\textwidth]{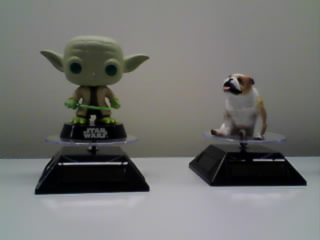}
        
        \includegraphics[width=1\textwidth]{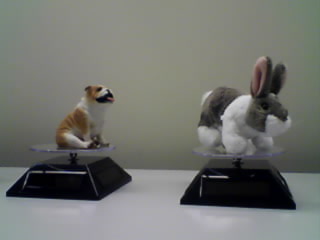}
    \end{subfigure}
    \caption{The experimental setup of the toy problem (left), and examples of images captured simultaneously by the two cameras (right).}
    \label{YodaBulldogRabbit}
\end{figure}

\subsection{Playing with Toys: Yoda, Bulldog and Rabbit}
In this experiment, we consider the toy problem presented in \cite{lederman2015learning}. 
The setting of the problem includes three objects: a figure of Yoda (green alien), a Bulldog, and a Rabbit, which were placed on rotating platforms. The three figures rotate in different speeds, and one (Yoda) in a different direction. This entire scene was captured by two cameras, as demonstrated in Figure \ref{YodaBulldogRabbit}~(left). The view of the first camera included both the figures of Yoda and Bulldog (Figure \ref{YodaBulldogRabbit}~(top-right)), while the view of the second camera included the Bulldog and the Rabbit (Figure \ref{YodaBulldogRabbit}~(bottom-right)). The two cameras were synchronized, i.e., they were taking simultaneous snapshots. In this problem, the latent variables are the orientation angles of the three figures, where the angle of the Bulldog is the common variable $\X$, and the angles of Yoda and of the Rabbit are the sensor specific-variables $\Y$ and $\Z$, respectively. The data at hand consist of images (snapshots) of the rotating figures, captured simultaneously by the two cameras. 

In \cite{lederman2015learning}, it was shown that the \ac{AD} algorithm attains a parametrization of the angle of the Bulldog $\X$, namely, the common variable hidden in the sets of images.
In this work we infer a parametrization of the angle of the sensor-specific Yoda $\Y$. 
In Figure \ref{yoda} we present the result of applying Algorithm \ref{alg1} in this setup where the size of the neighborhoods in the common variable domain is $q=15$. We scatter plot the first two coordinates in the obtained parametrization, $\hat{\y}$, and observe that the parametrization takes the shape of a circle, correctly representing a rotating angle.
To show that this angle is indeed associated with the rotating angle of Yoda, we overlay several images corresponding to the embedded points\footnotemark.
\footnotetext{We flipped the images captured by the camera horizontally for easier viewing of the figure.}
It can be seen that the orientation angle of Yoda corresponds to angles on the obtained circle.

\begin{figure}[t]
	\centering
	\includegraphics[width=0.75\textwidth]{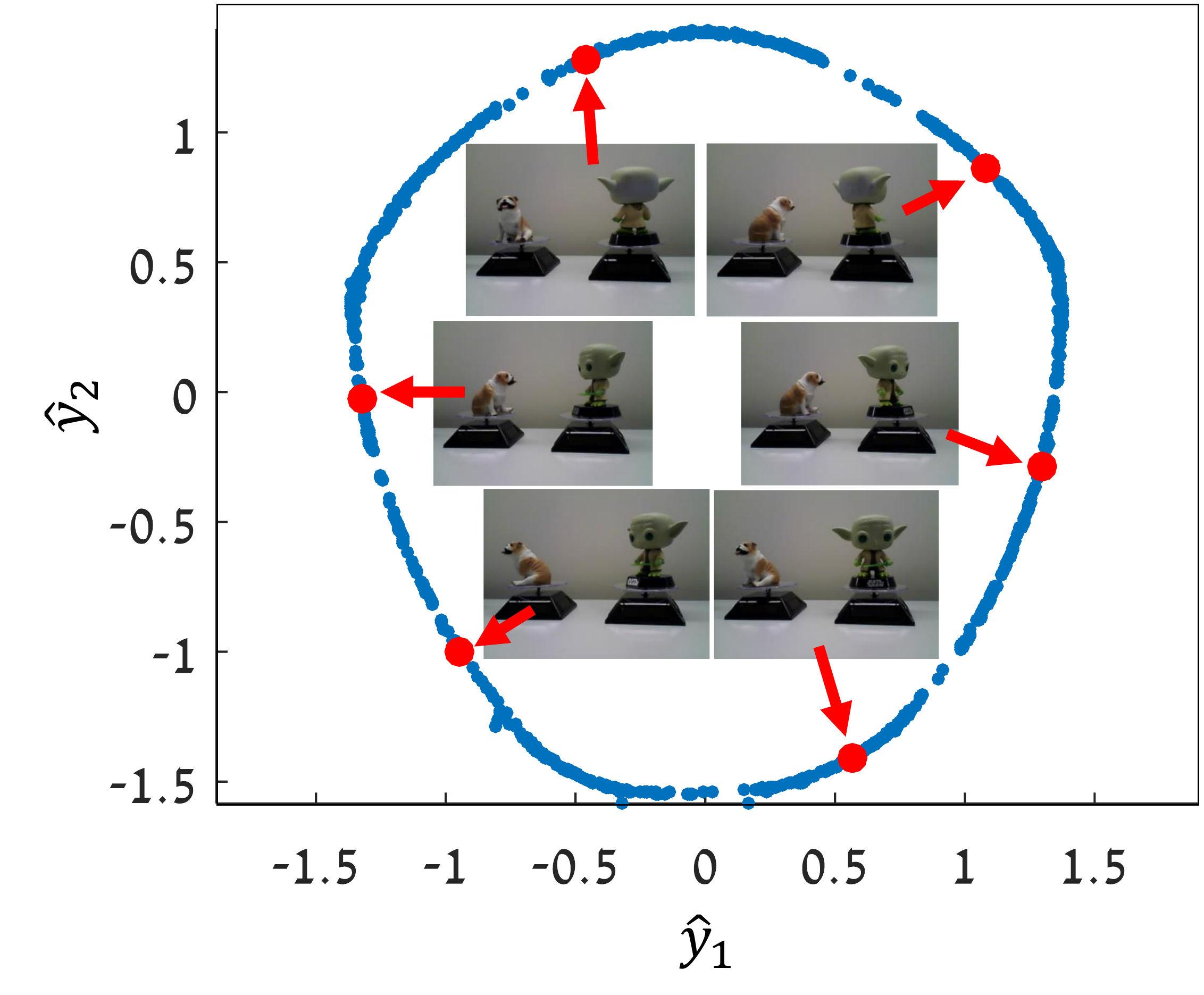}
	\caption{The parametrization obtained by the proposed algorithm, displaying the first two components from $\hat{\y}$. It is demonstrated that Algorithm \ref{alg1} enables us to capture the sensor-specific variable -- the angle of Yoda.}
	\label{yoda}
\end{figure}

\begin{figure}[t]
	\centering
	\setlength\figureheight{0.4\textwidth}
	\setlength\figurewidth{0.8\textwidth}
	\input{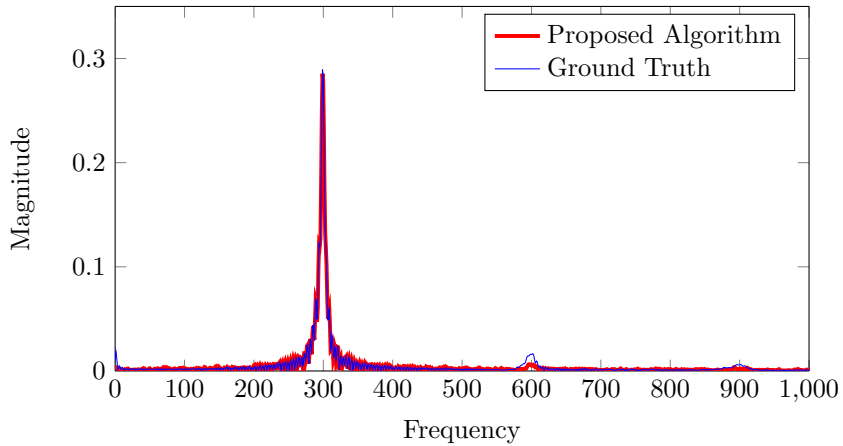}
	\caption{The discrete Fourier transform (magnitude only) of the first components in $\tilde{\y}$ and $\hat{\y}$.
	The parametrization $\tilde{\y}$ (ground truth) is obtained by diffusion maps applied to the cropped images, and the parametrization $\hat{\y}$ is obtained by Algorithm \ref{alg1}.}
	\label{dft}
\end{figure}

For further evaluation, we compute the ground truth parametrization of the angle of Yoda, denoted by $\tilde{\y}$. To this end, we crop all images captured by the first camera, discarding Bulldog and maintaining only Yoda, and then we apply diffusion maps to the set of cropped images.
We emphasize that the information of how to appropriately crop the image is not available to the proposed algorithm, which does not use any prior knowledge on the experimental setting, and it was done for illustration and evaluation purposes only. 
In Figure \ref{dft}, we present the discrete Fourier transform of both the first components of $\tilde{\y}$ and $\hat{\y}$; the parametrization $\tilde{\y}$ is obtained by diffusion maps applied to the cropped images, and the parametrization $\hat{\y}$ is obtained by Algorithm \ref{alg1}. We observe a sharp peak around the frequency $310$, which according to \cite{lederman2015learning} corresponds exactly to the rotation speed of Yoda\footnotemark.
\footnotetext{The frequency is given in terms of the number of cycles completed in the duration of the experiment.}
Moreover, the curves are similar, implying on the successful recovery of the sensor-specific variable by the proposed algorithm.

\subsection{Non-Invasive Fetal ECG}

Fetal heart rate monitoring \cite{freeman2012fetal,rochard1976nonstressed,li2016extract} is widely-used for the assessment of the fetal health both during pregnancy and during delivery. The most accurate method, relying on the placement of electrodes on the fetal scalp, carries many risks. Consequently, non-invasive measurements are usually carried out by placing electrodes on the abdomen of the mother. Naturally, the measured signal contains, in addition to the fetal's heart beats, the maternal \ac{ECG}, masking the desired information. In order to suppress the maternal \ac{ECG} and to extract the fetal \ac{ECG}, another (reference) electrode is often placed on the mother's thorax for the purpose of measuring only the maternal \ac{ECG}.

In practice, in addition to being occluded by the maternal \ac{ECG}, the fetal \ac{ECG} is also contaminated by noise. Power line disturbance and maternal muscle movements (electromyographic activity in the abdomen and uterus muscles of the mother) are only two typical examples of the possible interferences hindering the extraction of the fetal \ac{ECG} \cite{ferrara1982fetal}.

As reported in \cite{ferrara1982fetal}, due to its time-varying statistical character, the \ac{ECG} of the fetal is a highly nonstationary signal. Moreover, the relation between the measured abdomen signal and the fetal \ac{ECG} is arguably nonlinear. As such, standard approaches, e.g., the adaptive least mean squares (LMS) algorithm \cite{widrow1976stationary}, provide only coarse estimations in recovering the fetal \ac{ECG}, and the solution for this problem is not trivial, and it is still considered an open problem.

In \cite{ferrara1982fetal}, the authors suggest to tackle the fetal \ac{ECG} extraction problem by first extracting the maternal \ac{ECG} from the two measurements using adaptive noise cancelers. Then, given the result, that is the fetal \ac{ECG} plus muscle noise, the authors suggest to employ an adaptive signal enhancer in order to extract the fetal \ac{ECG} and attenuate the remaining noise. In particular, this adaptive signal enhancer relies on the alignment in time of an ensemble of similar pulses and the extraction of their statistics. The solution presented in that work is specifically-tailored for fetal \ac{ECG} extraction as it requires, for example, the detection of the peaks in the fetal \ac{ECG} signal using some peak detector. 
In contrast, in the sequel, we show that our approach does not require any knowledge about the task at hand.

We use the fetal \ac{ECG} extraction problem as a testbed for our algorithm not only to demonstrate its applicability, but also to show the relevance of the problem setting we present in this paper to real measured data. 
Let us now return to the problem formulation, as defined in Section \ref{Sec:ProblemFormulation}. In our context, the variable common to both the abdomen and thorax signals is the maternal \ac{ECG}, while the sensor-specific variable in the abdomen signal is the desired fetal \ac{ECG}.

We demonstrate that our proposed method is capable of not only recovering the maternal \ac{ECG} (common variable), but also factoring out the mother's pulse from the measured abdomen signal. This results in revealing the fetal \ac{ECG} (sensor-specific variable), which is relatively weak when compared to the maternal signal. More specifically, we show that our method builds a parameterization of the fetal \ac{ECG}, which, in turn, could aid in detecting fetal QRS complexes which are used in measuring the fetal's heart rate.

\begin{figure}[t]
	\centering
	\includegraphics[width=0.6\textwidth,height=0.4\textwidth]{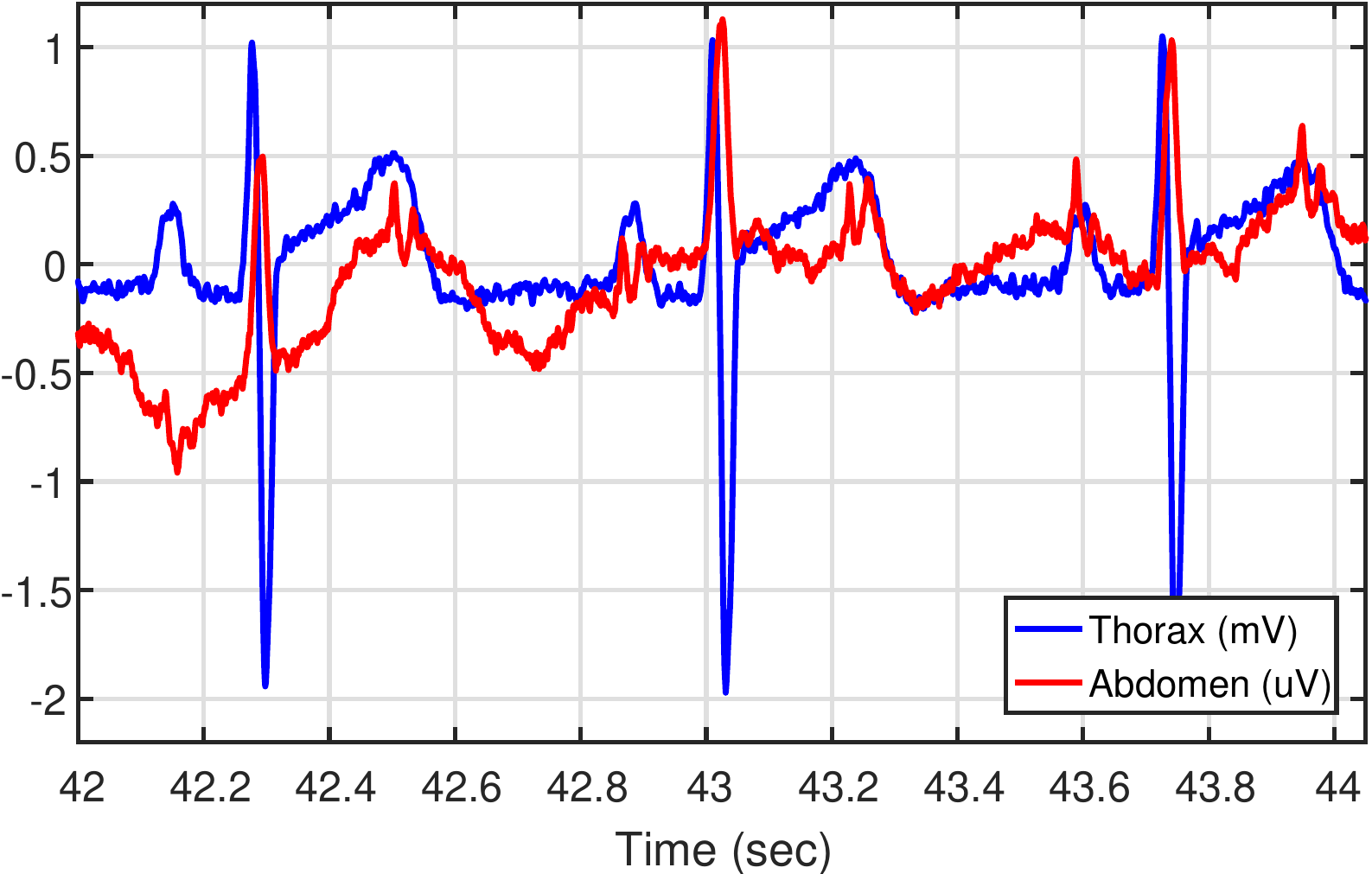}
	\caption{A short interval of the abdomen and thorax signals. The large spikes correspond the QRS complexes of the maternal \ac{ECG}, and the small spikes in the abdomen signal correspond to the QRS complexes of the fetal ECG.}
	\label{FECG_signal}
\end{figure}

\begin{figure}[t]
	\centering
	\begin{subfigure}[h]{0.5\textwidth}
		\centering
		\includegraphics[width=1\textwidth]{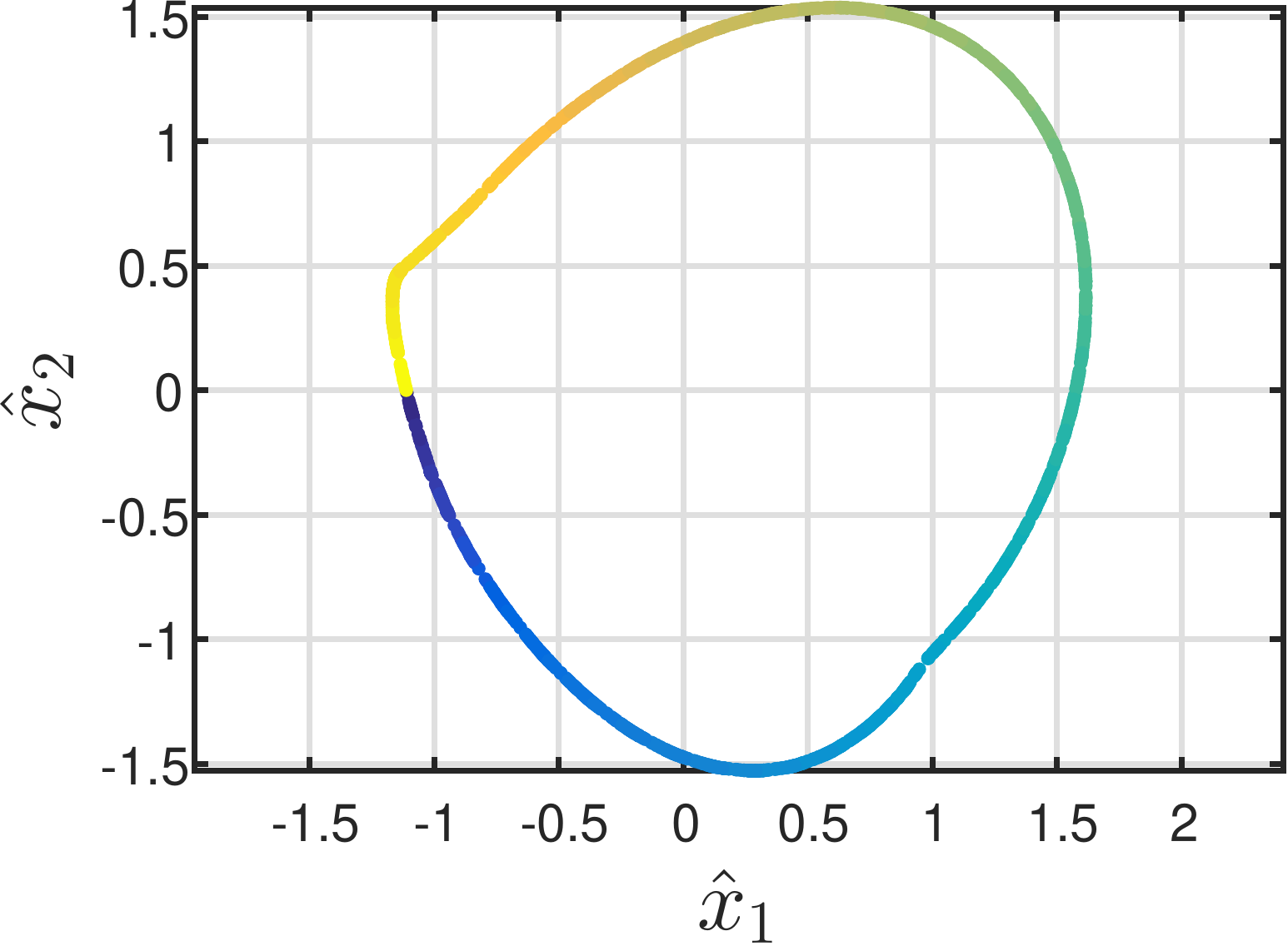}
	\end{subfigure}
	\caption{The parametrization obtained by taking the first two components from the AD algorithm applied to the abdomen and thorax signals.}
	\label{FECG_common}
	\centering
	\begin{subfigure}[h]{0.85\textwidth}
		\centering
		\includegraphics[width=1\textwidth]{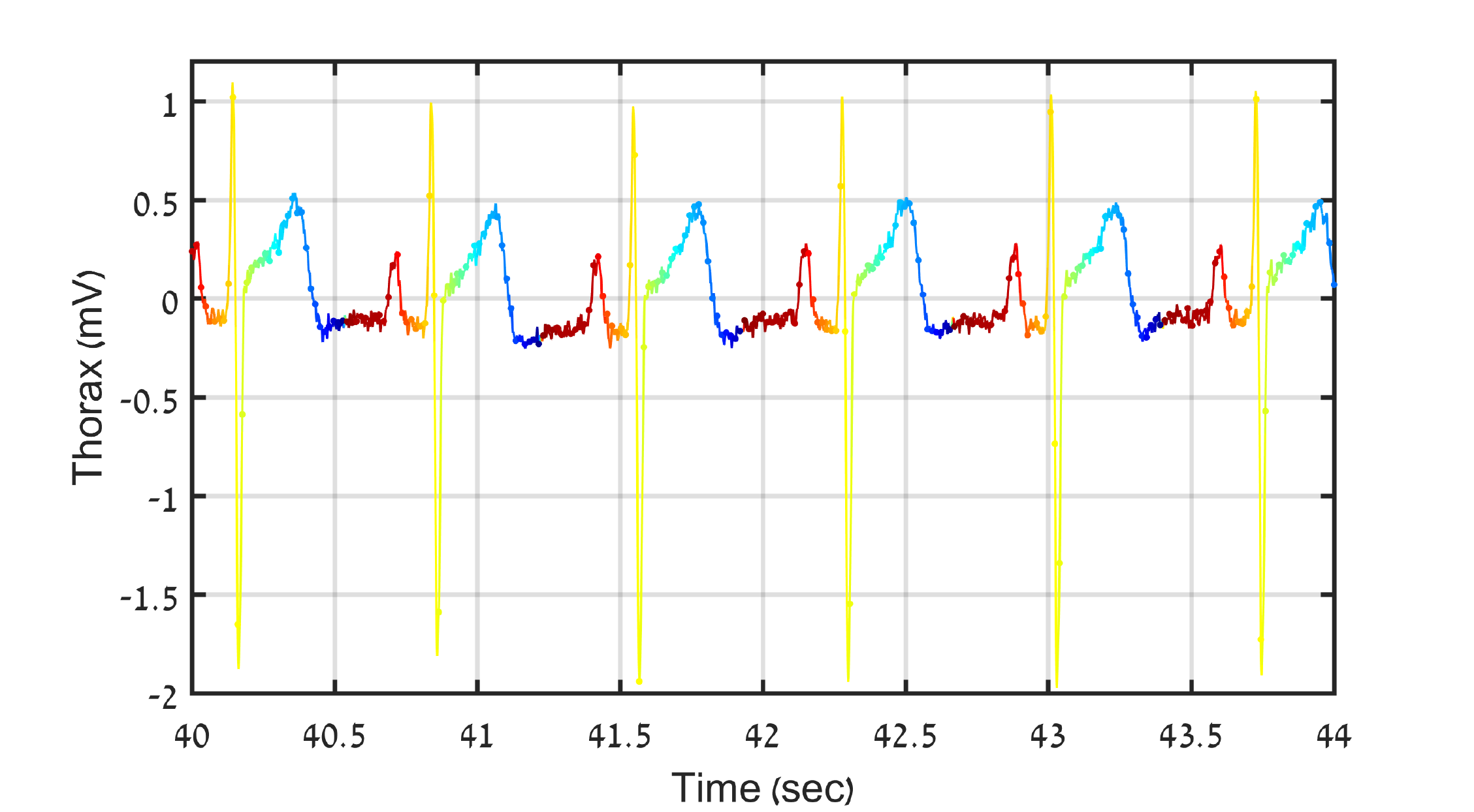}
	\end{subfigure}
	\caption{The thorax signal colored according to the common variable.}
	\label{FECG_common2}
\end{figure}

In our experiments we use the ``Non-invasive Fetal ECG Database'' from PhysioNet \cite{PhysioNet}, and apply our algorithm to the raw data without any preprocessing. In Figure \ref{FECG_signal} we present a short $2$ second interval from both the thorax and the abdomen signals. The length of the entire signals is $5$ minutes and $20$ seconds and in the following experiments we operate on a sample extracted from them that is of length $32.76$ seconds. The sampling rate of the \ac{ECG} signals is $1$~kHz.

Given these two signals, we first apply the \ac{AD} algorithm to extract the common variable and present the obtained parametrization in Figure \ref{FECG_common}. The algorithm is applied to time segments of length $256$ samples (lag map) with $16$ samples overlap, which are extracted from both the abdomen and thorax signals\footnotemark
\footnotetext{The mean of every segment was subtracted.}. 
Since the sampling rate is $1$ kHz, each segment is of duration $256$ miliseconds.
In the context of this paper, these segments are viewed as the sensor samples, and thus, are denoted by $\s_i^{(1)}$ and $\s_i^{(2)}$. 
Each $2$D point in the scatter plot in Figure \ref{FECG_common}, representing a pair of segments, is colored according to the angle created between the axis origin and the point itself (we choose this method of coloring due to the parametrization resembling a circle stemming from the signals periodicity). 
To emphasize the validity of our assumption that the common variable is indeed related to the maternal \ac{ECG}, we present in Figure \ref{FECG_common2} the thorax signal and color its samples according to the extracted parametrization of the common variable. Clearly, the common variable coincides with the cardiac cycle of the mother.

Next, we proceed by extracting the sensor-specific variable from the abdomen signal using our proposed algorithm. In this experiment we set the size of the neighborhoods in the common variable domain to be $q=21$. In Figure \ref{FECG_specific}, we present the abdomen signal colored according to the obtained parametrization of the sensor-specific variable. The results imply that indeed the sensor-specific variable is related to the \ac{ECG} of the fetal. Importantly, at $t=70.5$ sec, $t=77.8$ sec, and $t=88.4$, we observe that the fetal's heart beat is detected, even in pathological cases where it is completely ``buried'' in the maternal heart beat.

To demonstrate the generality of our method, we repeated the experiment and applied our algorithm to a signal measured from another patient. The results of this experiment are presented in Figure \ref{FECG_specific2}, showing that the parametrization of the sensor-specific variable manages to capture the ECG of the fetus in this case as well. Similarly, at times $t=69.9$ sec, $t=78.5$ sec, and $t=92.7$ sec the algorithm manages to capture the fetal ECG despite the significant occlusion by the maternal heart beat. This result also demonstrates cases in which the identification fails. For example, at times $t=78.2$ sec and $t=93.4$ sec a (possibly) redundant fetal heart peak is detected. One should note however that due to the lack of a ground truth, we can not be certain whether this is indeed a mis-identification. It might be the case that these are anomalies in the fetal heart rate and that the prediction is in fact correct.

To better support our results, we further apply our algorithm to two additional patients and depict the results in Figure \ref{FECG_specific3}. Here as well, it is observed that our algorithm manages to detect all the spikes corresponding to the QRS complexes of the fetus, exhibiting robustness to abnormal and uncharacteristic peaks, such as the one at time $t=43.8$ sec.

\begin{figure}[t]
	\centering
	\begin{subfigure}[t]{1\textwidth}
		\includegraphics[width=1\textwidth]{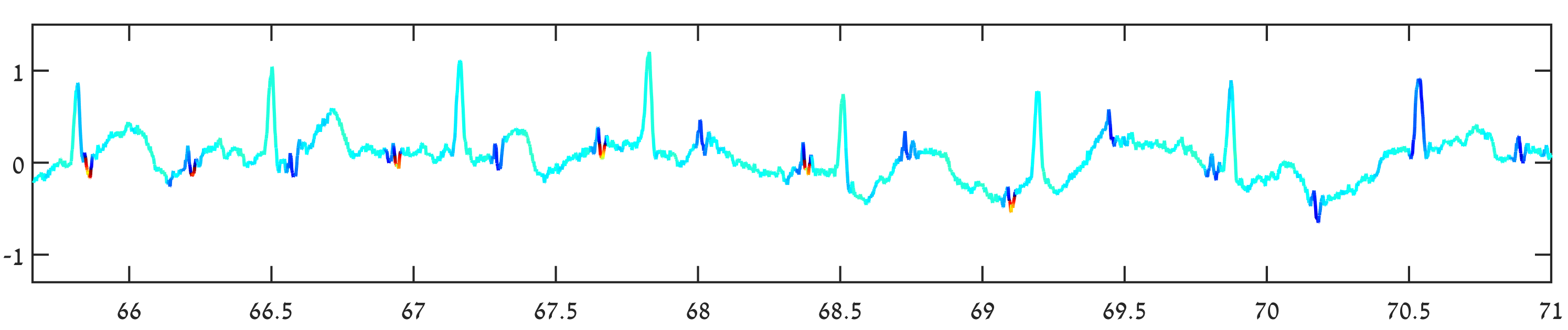}
	\end{subfigure}
	\qquad
	\centering
	\begin{subfigure}[t]{1\textwidth}
		\includegraphics[width=1\textwidth]{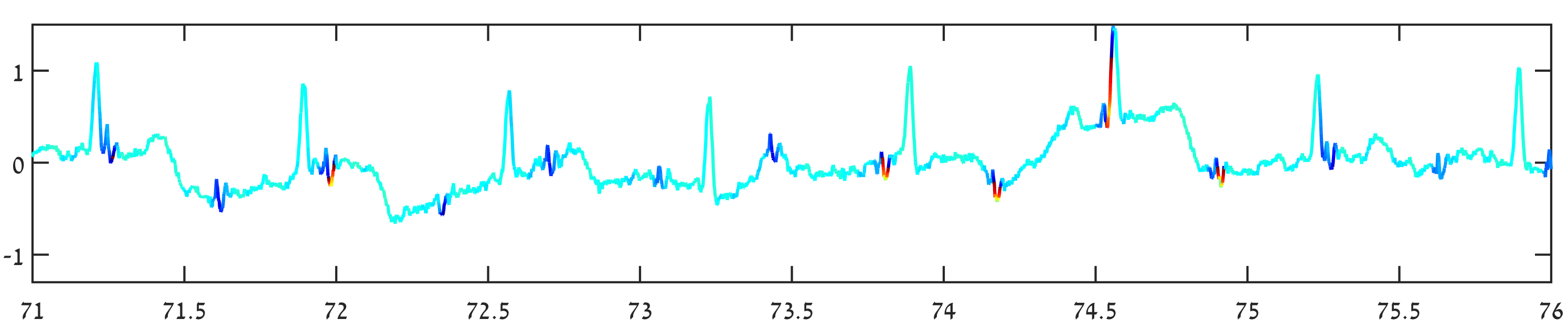}
	\end{subfigure}
	\qquad
	\centering
	\begin{subfigure}[t]{1.04\textwidth}
		\hspace{-0.58cm}
		\includegraphics[width=1\textwidth]{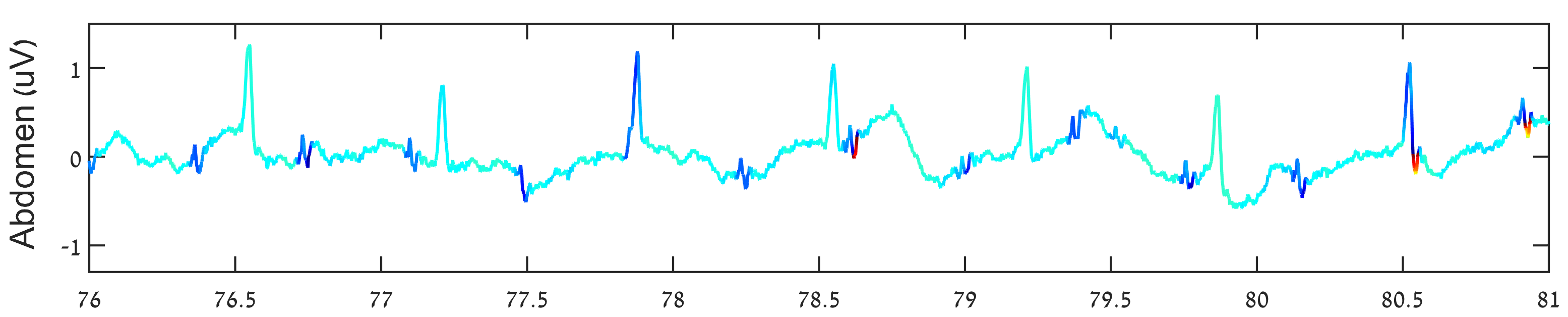}
	\end{subfigure}
	\qquad
	\centering
	\begin{subfigure}[t]{1\textwidth}
		\includegraphics[width=1\textwidth]{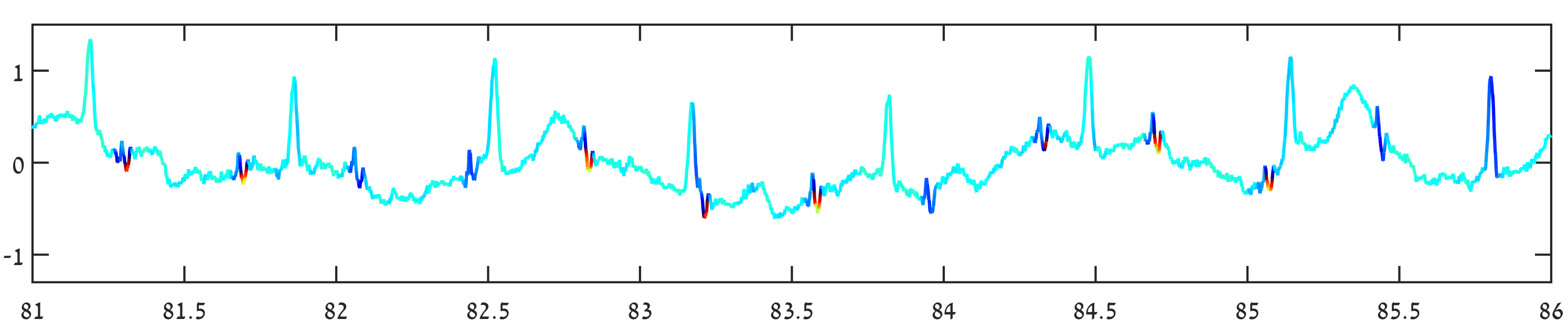}
	\end{subfigure}
	\qquad
	\centering
	\begin{subfigure}[t]{1\textwidth}
		\includegraphics[width=1\textwidth]{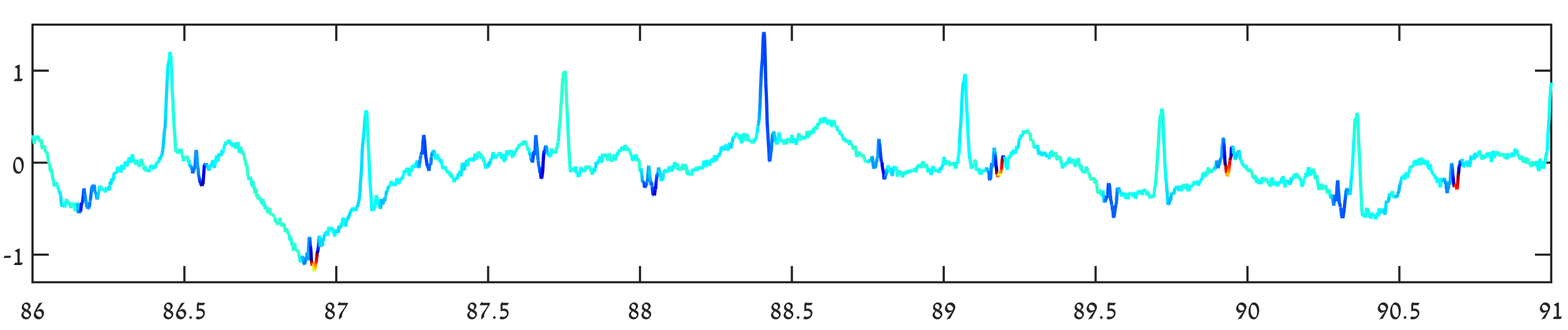}
	\end{subfigure}
	\qquad
	\centering
	\begin{subfigure}[t]{1\textwidth}
		\includegraphics[width=1\textwidth]{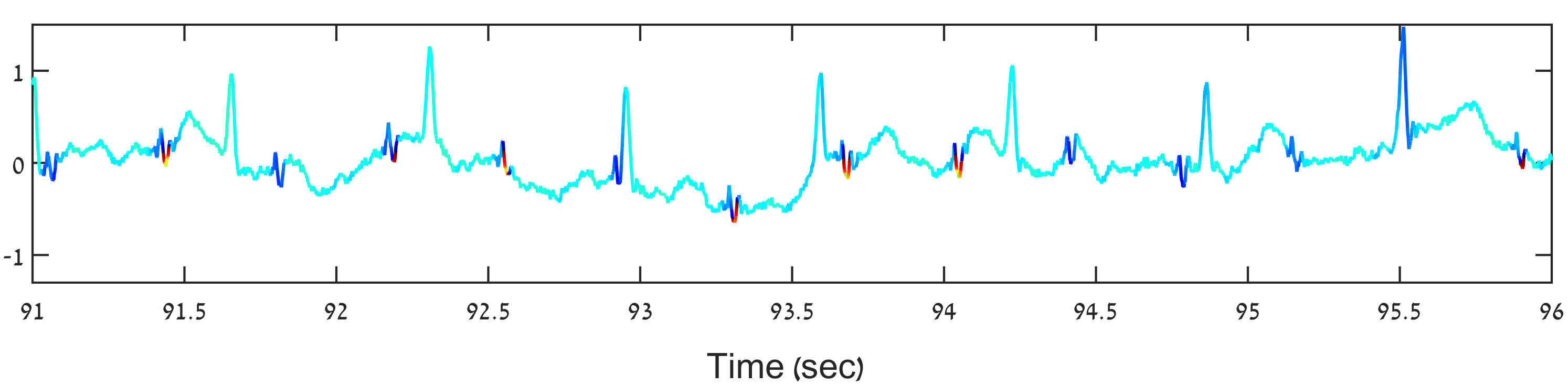}
	\end{subfigure}
	\qquad
	\vspace{-0.6cm}
	\caption{The abdomen signal as a function of time (in seconds). The signal is colored according to the extracted parametrization of the sensor-specific variable obtained by the proposed algorithm. Notice that the mother's peak at time $t=70.5$ sec, $t=77.8$ sec, and $t=88.4$ sec (as well as many other instances) completely hide the fetal's QRS. Nevertheless, our algorithm manages to detect it.}
	\label{FECG_specific}
\end{figure}

\begin{figure}[t]
	\centering
	\begin{subfigure}[t]{1\textwidth}
		\includegraphics[width=1\textwidth]{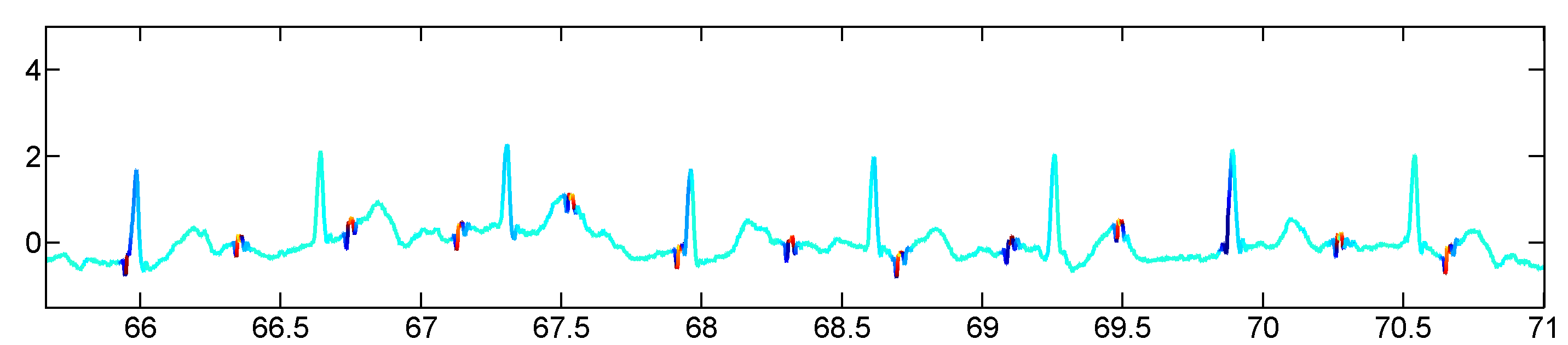}
	\end{subfigure}
	\qquad
	\centering
	\begin{subfigure}[t]{1\textwidth}
		\includegraphics[width=1\textwidth]{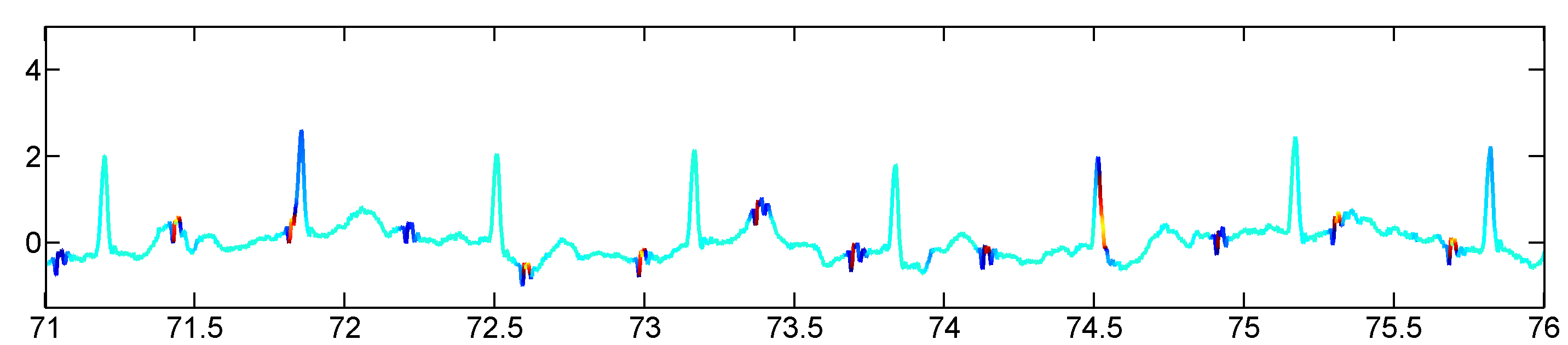}
	\end{subfigure}
	\qquad
	\centering
	\begin{subfigure}[t]{1\textwidth}
		\hspace{-0.58cm}
		\includegraphics[width=1.04\textwidth]{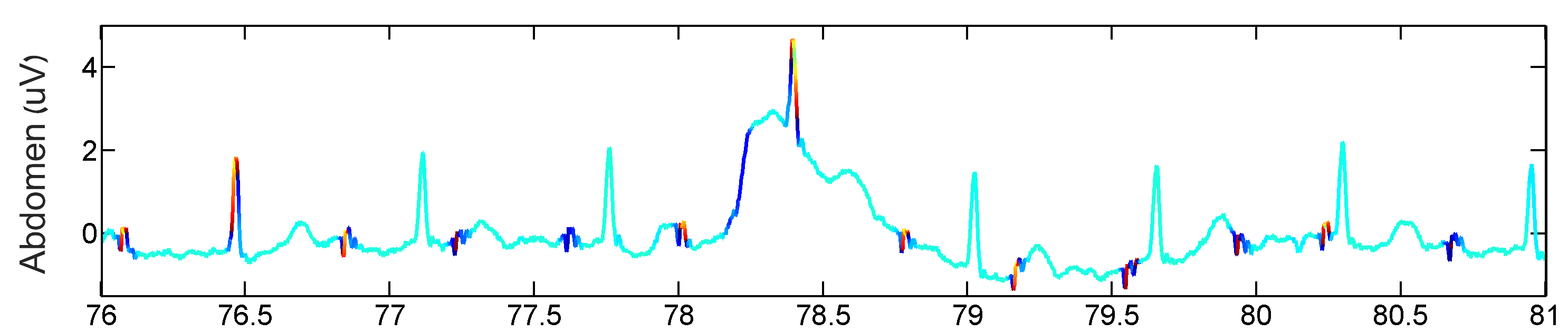}
	\end{subfigure}
	\qquad
	\centering
	\begin{subfigure}[t]{1\textwidth}
		\includegraphics[width=1\textwidth]{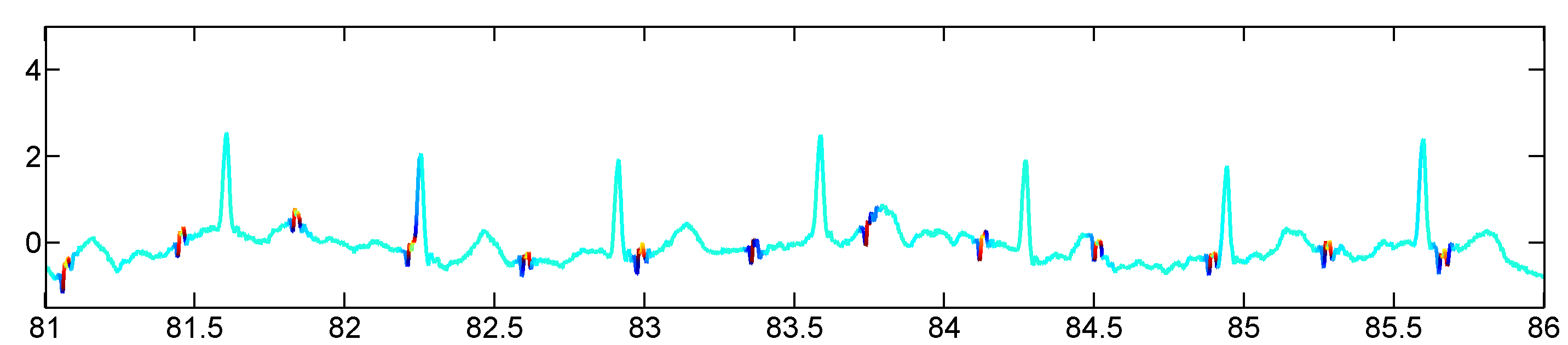}
	\end{subfigure}
	\qquad
	\centering
	\begin{subfigure}[t]{1\textwidth}
		\includegraphics[width=1\textwidth]{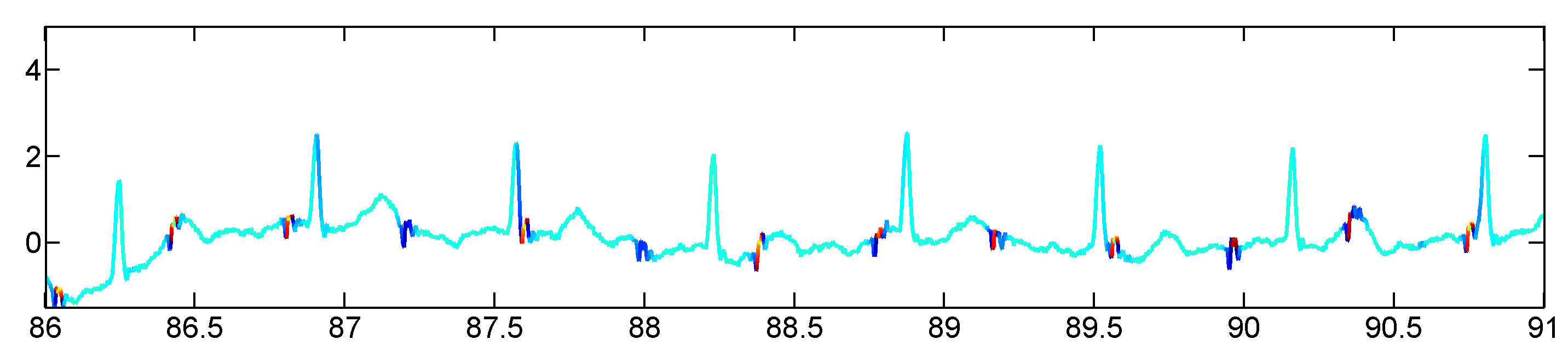}
	\end{subfigure}
	\qquad
	\centering
	\begin{subfigure}[t]{1\textwidth}
		\includegraphics[width=1\textwidth]{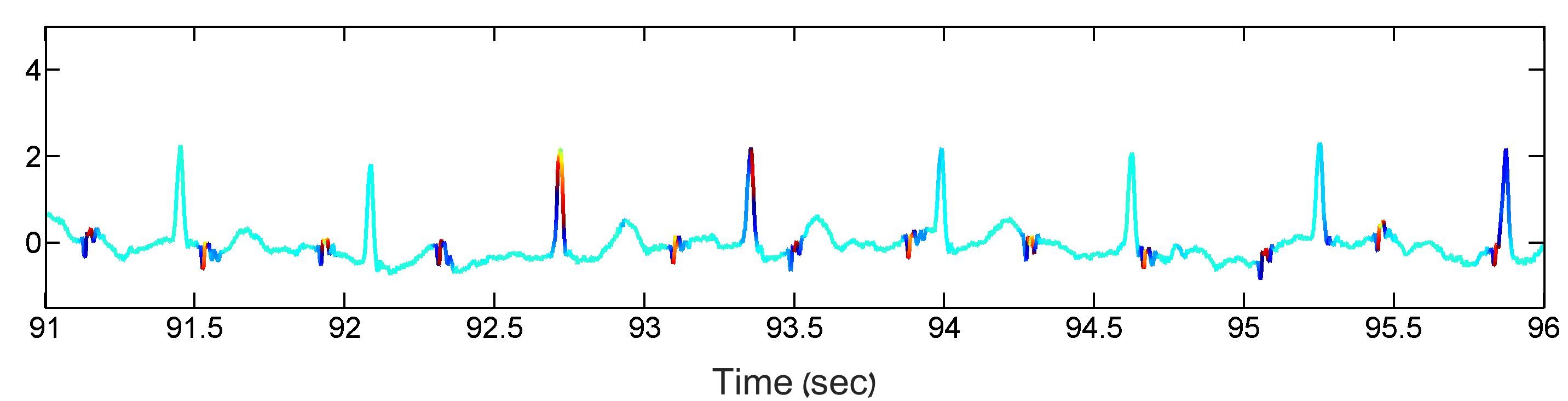}
	\end{subfigure}
	\qquad
	\vspace{-0.6cm}
	\caption{The abdomen signal of a second patient as a function of time (in seconds). Similarly to Figure \ref{FECG_specific}, the signal is colored according to the extracted parametrization of the sensor-specific variable obtained by the proposed algorithm. At times $t=69.9$ sec, $t=78.5$ sec, and $t=92.7$ sec the maternal peaks completely hide the fetal \ac{ECG}, yet, our algorithm enables to capture it. Moreover, at times $t=78.2$ sec and $t=93.4$ sec, we observe that a (possibly) redundant fetal peak is detected.}
	\label{FECG_specific2}
\end{figure}

\begin{figure}[t]
	\centering
	\begin{subfigure}[t]{1\textwidth}
		\includegraphics[width=1\textwidth]{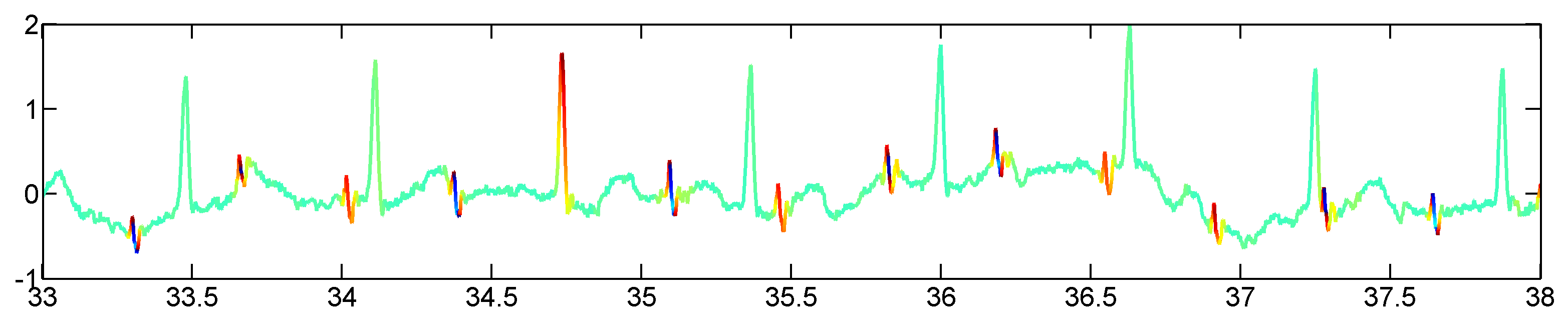}
	\end{subfigure}
	\qquad
	\centering
	\begin{subfigure}[t]{1\textwidth}
		\includegraphics[width=1\textwidth]{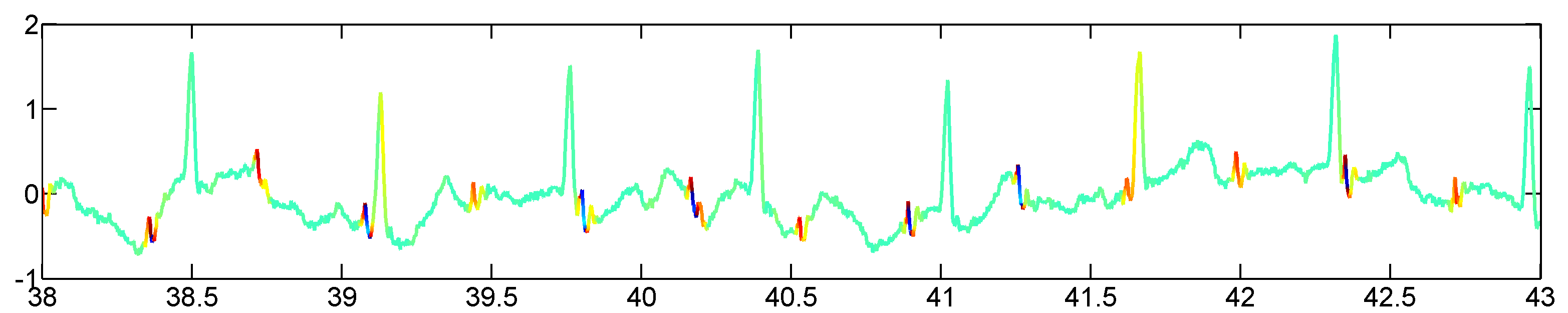}
	\end{subfigure}
	\qquad
	\centering
	\begin{subfigure}[t]{1.03\textwidth}
		\hspace{-0.57cm}
		\includegraphics[width=1\textwidth]{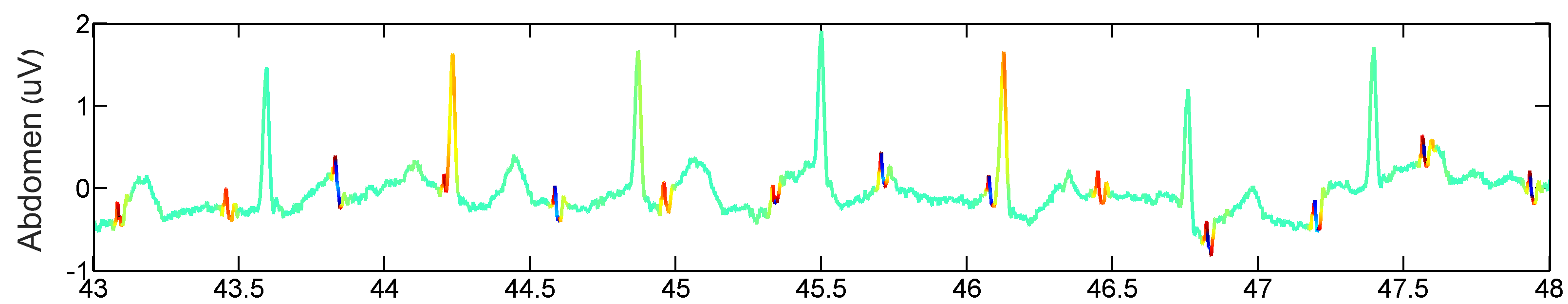}
	\end{subfigure}
	\qquad
	\centering
	\begin{subfigure}[t]{1\textwidth}
		\includegraphics[width=1\textwidth]{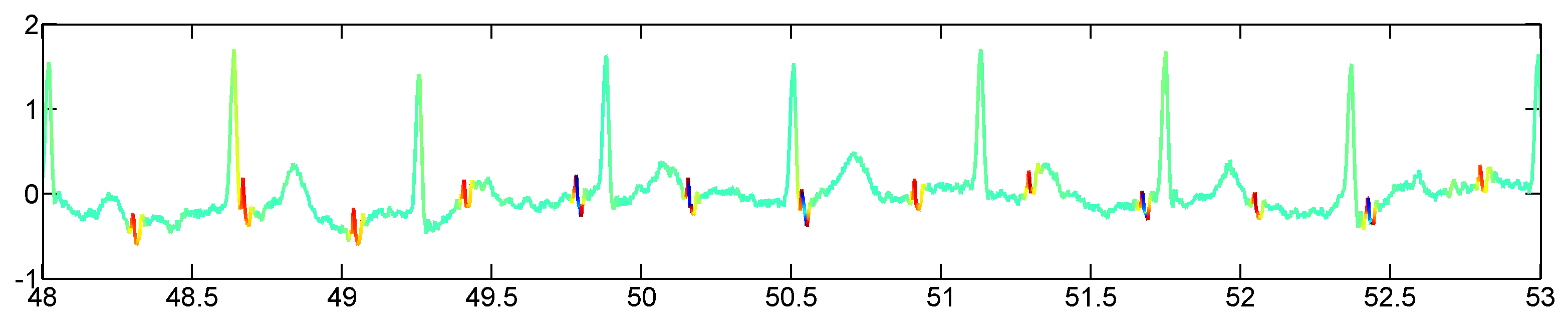}
	\end{subfigure}
	\qquad
	\centering
	\begin{subfigure}[t]{1\textwidth}
		\includegraphics[width=1\textwidth]{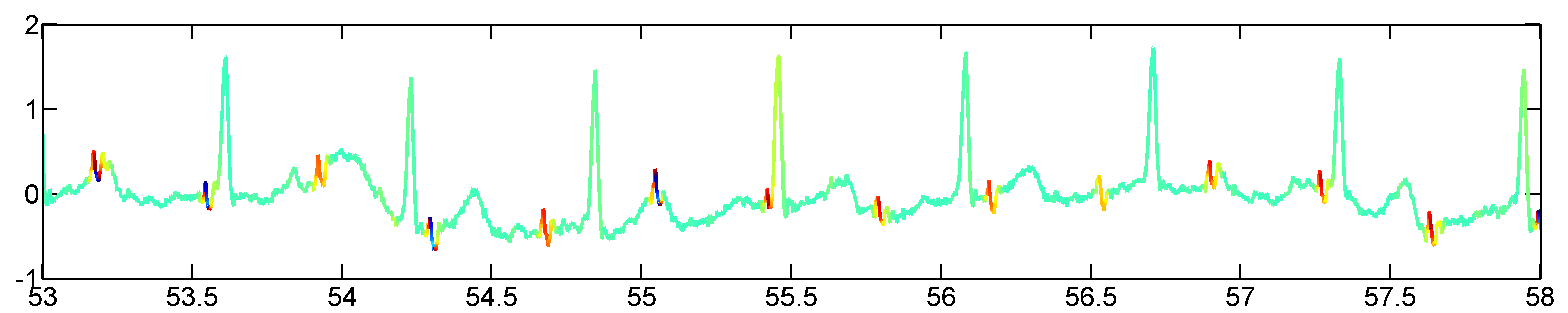}
	\end{subfigure}
	\qquad
	\centering
	\begin{subfigure}[t]{1\textwidth}
		\includegraphics[width=1\textwidth]{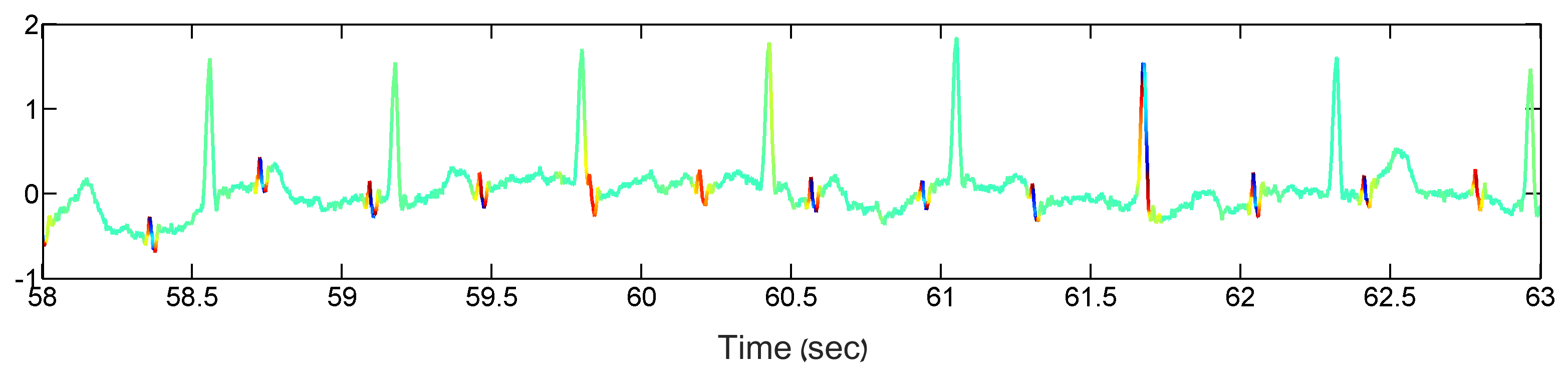}
	\end{subfigure}
	\qquad
	\vspace{-0.6cm}
	\caption{The abdomen signal of a third patient as a function of time (in seconds). The signal is colored similarly as in Figure \ref{FECG_specific} and Figure \ref{FECG_specific2}. Notice the uncharacteristic peak at time $t=43.8$ sec that is captured by our algorithm.}
	\label{FECG_specific3}
\end{figure}

We note that in order to establish baseline results, in addition to using the proposed algorithm, we have attempted to use simpler methods such ICA \cite{hyvarinen1999fast,hyvarinen2000independent} in order to separate the maternal \ac{ECG} from the fetal \ac{ECG}. Broadly, these methods attempt to transform the data at hand (after some whitening) into components that are as statistically independent from each other as possible. However, these methods did not obtain satisfactory results in separating the fetal and maternal \ac{ECG} signals, since some nontrivial preprocessing (in addition to whitening and dimensionality reduction) must be employed in order to facilitate their employment. Conversely, we emphasize that our proposed method does not rely on any preprocessing of the data (except for the segment mean subtraction).

\section{Conclusions}
\label{Sec:Conclusions}
Given a set of measurements, originating from several sensors, the \ac{AD} algorithm extracts a parametrization of a variable common to all sources. In this work, leveraging on \ac{AD}, we proposed a method which further analyzes the signals by extracting the sensor-specific variables. We provided a theoretical justification as well as various applications. A shortcoming of our method is the need to extract the intermediate common variable parametrization. Proposing a method that could skip this stage is a promising future direction.

\clearpage

\section*{References}

\bibliography{mybibfile}

\acrodef{PDF}{probability density function}
\acrodef{AD}{alternating diffusion}
\acrodef{ECG}{Electrocardiography}

\end{document}